\def\year{2022}\relax
\documentclass[letterpaper]{article} 
\usepackage{geometry}
\usepackage{aaai22}  
\usepackage{times}  
\usepackage{helvet}  
\usepackage{courier}  
\usepackage[hyphens]{url}  
\usepackage{graphicx} 
\usepackage{natbib}  
\usepackage{caption} 
\DeclareCaptionStyle{ruled}{labelfont=normalfont,labelsep=colon,strut=off} 
\usepackage{algorithm}
\usepackage{algcompatible}

\usepackage{newfloat}
\usepackage{listings}
\floatstyle{ruled}
\newfloat{listing}{tb}{lst}{}
\floatname{listing}{Listing}
\usepackage{amsmath}
\usepackage{amsthm}
\usepackage{amssymb}

\usepackage{bm}
\usepackage{fancyvrb}

\usepackage{booktabs}
\usepackage{multirow}

\usepackage{hyperref}

\newtheorem{theorem}{Theorem}

\newtheorem{corollary}[theorem]{Corollary}

\algnewcommand\algorithmicreturn{\textbf{return }}
\algnewcommand\RETURN{\State \algorithmicreturn}%

\DeclareMathOperator{\st}{s.t.}
\DeclareMathOperator{\argmin}{argmin}

\newcommand{\setR}{\mathbb{R}}

\newcommand{\vareps}{\varepsilon}
\newcommand{\vecOne}{\bm{1}}
\newcommand{\E}{\mathbb{E}}

\newcommand{\sprod}[2]{\langle{#1}, {#2}\rangle}
\newcommand{\norm}[1]{\left\Vert#1\right\Vert}

\newcommand{\ignore}[1]{}

\title{Optimization for Classical Machine Learning Problems on the GPU}
\author {
    S\"oren Laue\textsuperscript{\rm 1, 2},
    Mark Blacher\textsuperscript{\rm 1},
    Joachim Giesen\textsuperscript{\rm 1}
}
\affiliations {
    \textsuperscript{\rm 1} Friedrich-Schiller-University Jena, Germany\\
    \textsuperscript{\rm 2} Data Assessment Solutions GmbH\\
    \{soeren.laue, mark.blacher, joachim.giesen\}@uni-jena.de
}
\usepackage{bibentry}

\begin{document}

\maketitle
\thispagestyle{plain}
\pagestyle{plain}

\begin{abstract}
Constrained optimization problems arise frequently in classical machine learning. There exist frameworks addressing constrained optimization, for instance, CVXPY and GENO. However, in contrast to deep learning frameworks, GPU support is limited. Here, we extend the GENO framework to also solve constrained optimization problems on the GPU. The framework allows the user to specify constrained optimization problems in an easy-to-read modeling language. A solver is then automatically generated from this specification. When run on the GPU, the solver outperforms state-of-the-art approaches like CVXPY combined with a GPU-accelerated solver such as cuOSQP or SCS by a few orders of magnitude.
\end{abstract}

\section{Introduction}


Training classical machine learning models typically means solving an optimization problem. Hence, the design and implementation of solvers for training these models has been and still is an active research topic. While the use of GPUs is standard in training deep learning models, most solvers for classical machine learning problems still target CPUs.  Easy-to-use deep learning frameworks like TensorFlow or PyTorch can be used to solve unconstrained problems on the GPU. However, many classical problems entail constrained optimization problems. So far,  deep learning frameworks do not support constrained problems, not even problems with simple box constraints, that is, bounds on the variables. Optimization frameworks for classical machine learning like CVXPY~\cite{AgrawalVDB18, DiamondB16} can handle constraints but typically address CPUs. Here, we extend the GENO framework~\cite{LaueMG2019} for constrained optimization to also target GPUs.

Adding GPU support to an optimization framework for classical machine learning problems is not straightforward. An efficient algorithmic framework could use the limited-memory quasi-Newton method L-BFGS-B~\cite{ByrdLNZ95} that allows to solve large-scale optimization problems with box constraints. More general constraints can then be addressed by the augmented Lagrangian approach~\cite{Hestenes69,Powell69}. However, porting the L-BFGS-B method to the GPU does not provide any efficiency gain for the following reason: In each iteration the method involves an inherently sequential Cauchy point computation that determines the variables that will be modified in the current iteration. The Cauchy point is computed by minimizing a quadratic approximation over the gradient projection path, resulting in a large number of sequential scalar computations, which means that all but one core on a multicore processor will be idle. This problem is aggrevated on modern GPUs that feature a few orders of magnitude  more cores than a standard CPU, e.g., 2304 instead of 18. 

Let us substantiate this issue by the following example (see the appendix for details). On a non-negative least squares problem, the  \mbox{L-BFGS-B} algorithm needs $4.9$ seconds in total until convergence, where $0.6$ seconds are spent in the Cauchy point subroutine. When run on a modern GPU, the same code needs $5.2$ seconds in total while $4.6$ seconds are spent in the Cauchy point subroutine. It can be seen that while all other parts of the L-BFGS-B algorithm can be parallelized nicely on a GPU, the inherently sequential Cauchy point computation does not and instead, dominates the computation time on the GPU; as a result, the \mbox{L-BFGS-B} method is not faster on a GPU than on a CPU, rendering the benefits of the GPU moot. 

Here, we present a modified \mbox{L-BFGS-B} method that runs efficiently on the GPU, because it avoids the sequential Cauchy point computation. For instance, on the same problem as above, it needs $0.8$ seconds in total on the GPU.
We integrate an implementation of this method into the GENO framework for constrained optimization. Experiments on several classical machine learning problems show that our implementation outperforms state-of-the-art approaches like the combination of CVXPY with GPU-accelerated solvers such as cuOSQP or SCS by several orders of magnitude.

\textbf{Contributions.}
The contributions of this paper can be summarized as follows:\\
1. We design a provably convergent L-BFGS-B algorithm that can handle box constraints on the variables and runs efficiently on the GPU. \\
2. We combine our algorithm with an Augmented Lagrangian approach for solving constrained optimization problems. \\
3. We integrate our approach 
into the GENO framework for constrained optimization.
It outperforms comparable state-of-the-art approaches by several orders of magnitude.


\section{State of the Art}

Optimization for classical machine learning is either addressed by problem specific solvers, often wrapped in a library or toolbox, or by frameworks that combine a modeling language with a generic solver. The popular toolbox scikit-learn~\cite{scikit-learn} does not provide GPU support yet. 
Spark~\cite{ZahariaXWDADMRV16} has recently added GPU support and hence, enabling GPU acceleration for some algorithms from its machine learning toolbox MLlib~\cite{mllib}. The modeling language CVXPY~\cite{AgrawalVDB18, DiamondB16} can be paired with solvers like cuOSQP~\cite{SchubigerBL20} or SCS~\cite{ocpb:16, scs} that provide GPU support. 

In contrast to deep learning, classical machine learning can involve constraints, which poses extra algorithmic challenges. There are several algorithmic approaches for dealing with constraints that could be adapted for the GPU. As mentioned before, we have decided to adapt the \mbox{L-BFGS-B} method~\cite{ByrdLNZ95}. Its original Fortran-code~\cite{ZhuBLN97} is still the predominant solver in many toolboxes like scikit-learn or scipy for solving box-constrained optimization problems. In the following, we briefly describe algorithmic alternatives to the L-BFGS-B method and argue why we decided against using them for our purposes.   

Proximal methods, including alternating direction method of multipliers (ADMM) have been used for solving constrained and box-constrained optimization problems. The literature on proximal methods is vast, see for instance,~\cite{boydADMM, ParikhB14} for an overview. Prominent examples that relate to our work are OSQP~\cite{StellatoBGBB20} and SCS~\cite{ocpb:16, scs}. Unfortunately, both methods require a large number of iterations. One approach to mitigate this problem is to keep the penalty parameter $\rho$, which ties the constraints to the objective function fixed for each iteration. Then, the Cholesky-decomposition of the KKT system needs to be computed only once and can be reused in subsequent iterations, leading to a slow first iteration but very fast consecutive iterations. 
The OSQP solver has been shown to be the fastest general purpose solver for quadratic optimization problems~\cite{StellatoBGBB20}, beating commercial solvers like Gurobi as well as Mosek. Gurobi as well as Mosek do not provide GPU support and so far, there is no intention to do so in the near future~\cite{gurobi}. OSQP does not support GPU acceleration either. However, it has been ported to the GPU as cuOSQP~\cite{SchubigerBL20}, where it was shown to outperform its CPU version by an order of magnitude.  

The simplest way to solve box-constrained optimization problems is probably the projected gradient descent method~\cite{Nesterov04}. However, it is as slow as gradient descent and not applicable to practical problems. Hence, there have been a number of methods proposed that try to combine the projection approach with better search directions. For instance, \cite{Berg20} applies L-BFGS updates only to the currently active face. If faces switch between iterations, which happens in almost all iterations, it falls back to standard spectral gradient descent. A similar approach is the non-monotone spectral projected gradient descent approach as described in~\cite{SchmidtKS11}. It also performs backtracking along the projection arc and cannot be parallelized efficiently. Another variant solves a sub-problem in each iteration that is very expensive and hence, only useful when the cost of computing the function value and gradient of the original problem is very expensive~\cite{SchmidtBFM09}, which is typically not the case for standard machine learning problems. Another approach for solving box-constrained optimization problems has been described in~\cite{KimSD10}. However, it is restricted to strongly convex problems. For small convex problems, a projected Newton method has been described in~\cite{Bertsekas82}.

Nesterov acceleration~\cite{Nesterov83} has also been applied to proximal methods~\cite{BeckT09, LiL15}. However, similar to Nesterov's optimal gradient descent algorithm~\cite{Nesterov83}, one needs several Lipschitz constants of the objective function, which are usually not known. Quasi-Newton methods do not need to know such parameters and have been shown to perform equally well or even better. 
Convergence rates have been obtained for quasi-Newton methods on special instances, e.g., quadratic functions with unbounded domain. Recently, improved convergence rates have been proved in~\cite{RodomanovN21} for the unbounded case.

\section{Algorithm} \label{sec:algo}

\ignore{
Here, we present our extension of the L-BFGS algorithm that can additionally handle box constraints and 
runs efficiently on the GPU. In the algorithm and its analysis we are using the following notation. 
}

Here, we present our extension of the L-BFGS algorithm for minimizing a differentiable function $f\colon\setR^n\to\setR$, which can additionally handle box constraints on the variables, i.e., $l\leq x\leq u$, $l, u\in\setR^n\cup \{\pm\infty\}$, and which runs efficiently on the GPU.

\paragraph{Notation.} A sequence of scalars or vectors is denoted by upper indices, e.g., $x^1,\ldots,x^k$. The projection of a vector $x\in\setR^n$ to the coordinate set $S\subseteq\{1, \ldots n\}=:[n]$ is denoted as $x[S]$. If $S$ is a singelton $\{i\}$, then $x[S]$ is just the $i$-th coordinate $x_i$ of $x$. Similarly, the projection of a square matrix $B$ onto the rows and columns in an index set $S$ is denoted by $B[S, S]$. Finally, the Euclidean norm of $x$ is denoted by $\norm{x}$, and the corresponding scalar product between vectors $u$ and $v$ is denoted as $\sprod{u}{v}$. 

\begin{algorithm}[h!]
   \caption{GPU-efficient L-BFGS-B Method}
   \label{algo:1}
\begin{algorithmic}[1]
   \STATEx {\hspace{-0.5cm} \bfseries Input:} initial iterate $x^0$ with $l\leq x^0\leq u$
   \vspace{0.2cm}
   
   \STATE {\bfseries Initialization:} set $k\leftarrow 0$
   \REPEAT
   \STATE compute $\nabla f(x^k)$ and set $S^k$ of free variables (Eq.~\ref{eq:ws})
   \STATE solve $-\nabla f(x^k)[S^k]=B^k[S^k, S^k] d^k[S^k]$ using\newline
   \hspace*{2.5ex}\mbox{L-BFGS} modified two-loop algorithm, see appendix \label{line:0}
   \STATE set $d^k[\bar S^k] = 0$
   \STATE $p^k$ = projectDirection$(x^k, \nabla f(x^k), d^k)$
   \STATE $x^{k+1} = x^k + \alpha^k p^k$ using line search with appropriate \newline 
   \hspace*{2.2ex} upper bound on $\alpha^k$
   \STATE $y^k=\nabla f(x^{k+1}) - \nabla f(x^k)$ and $s^k=x^{k+1}-x^{k}$
   \STATE store new curvature pair $(y^k, s^k)$ 
   \STATE set $k\leftarrow k+1$
   \UNTIL{converged}
\end{algorithmic}
\end{algorithm}

\begin{algorithm}[h!]
   \caption{projectDirection$(x^k, \nabla f(x^k), d^k)$}
   \label{algo:2}
\begin{algorithmic}[1]
  \STATE compute $z^k = x^k+d^k$ and project $z^k$ onto feasible region
  \STATE compute $p^k = z^k - x^k$
  \IF{$\sprod{p^k}{\nabla f(x^k)} \leq -\vareps \|p^k\|^2$ and $\|p^k\|^2 \geq \vareps$} \label{line:1}
  \RETURN $p^k$
  \ELSE
  \STATE set $p^k = d^k$ \label{line:2}
  \STATE set $(p^k)_i = 0 \, \,\, \forall i$ with $ (d^k)_i < 0$ and $ (x^k)_i \leq l_i + \vareps$
  \STATE set $(p^k)_i = 0 \, \,\, \forall i$ with $ (d^k)_i > 0$ and $ (x^k)_i \geq u_i - \vareps$ \label{line:3}
  \RETURN $p^k$
  \ENDIF
\end{algorithmic}
\end{algorithm}

Like the original  L-BFGS-B algorithm, our extension runs in iterations until a convergence criterion is met. Likewise, it distinguishes between fixed and free variables in each iteration, i.e., variables that are fixed at their boundaries and variables that are optimized in the current iteration. In contrast to the original L-BFGS-B algorithm, we can avoid the inherently sequential Cauchy point computation by determining the fixed and free variables directly. Given $\varepsilon >0$, we compute in iteration $k$ the index set (working set)
\begin{equation}
  \label{eq:ws}
\begin{aligned}
S^k = [n] \setminus \big ( & \big\{ i \ |\  (x^k)_i \leq l_i + \vareps \mbox{ and } \nabla f(x^k)_i \geq 0\big\} \\ 
\cup \, & \big\{ i \ |\  (x^k)_i \geq u_i - \vareps \mbox{ and } \nabla f(x^k)_i \leq 0\big\} \big )
\end{aligned}
\end{equation}
of free variables. Here, $\{i \ |\  (x^k)_i \leq l_i + \vareps \mbox{ and } \nabla f(x^k)_i \geq 0\}$ holds the indices of optimization variables that, at iteration $k$, are within an $\vareps$-interval of the lower bound. Analogously, $\{i \ |\  (x^k)_i \geq u_i - \vareps \mbox{ and } \nabla f(x^k)_i \leq 0\}$ holds all indices, where the optimization variable is close to the upper bound. The complement of $S^k$, i.e., the index set of all non-free (fixed) variables is denoted by $\bar{S^k}$. Algorithm~\ref{algo:1} computes the quasi-Newton search direction $d^k[S^k]$ only on the free variables (Line~\ref{line:0}). It then projects this direction onto the feasible set using Algorithm~\ref{algo:2}. If it is a feasible descent direction, it takes a step into this direction. Otherwise, it takes a step into the original quasi-Newton direction until it hits the boundary of the feasible region. While the original \mbox{L-BFGS-B} algorithm uses a line search with quadratic and cubic interpolation to satisfy the strong Wolfe conditions, we observed that this does not provide any benefit for optimization problems from machine learning. Hence, our implementation uses a simple backtracking line search to satisfy the Armijo condition~\cite{NocedalW99}. Even when the function is convex and satisfies the curvature condition for all variables, it does not necessarily satisfy the curvature condition for the set of free variables with indices in $S^k$. Hence, satisfying the strong Wolfe conditions is not necessary and instead the curvature condition is checked for the current set of free variables in the modified two-loop Algorithm~3 (see the appendix). The following theorem asserts that our algorithm converges to a stationary point.

\begin{theorem}\label{thm:1}
Let $f$ be a differentiable function with an $L$-Lipschitz continuous gradient. 
If $f$ is bounded from below, then Algorithm~\ref{algo:1} converges to a feasible stationary point. 
\end{theorem}

\begin{proof}
We have for any differentiable 
function with $L$-Lipschitz continuous gradient that
\[
  f(x+\alpha p) \leq f(x) + \sprod{\nabla f(x)}{\alpha p} + \frac{L}{2} \|\alpha p\|^2
\]
see, e.g.,~\cite{Nesterov04}. For computing the search direction $p^k$, we distinguish two cases in Algorithm~\ref{algo:2}. In the first case, we have 
\[
\sprod{p^k}{\nabla f(x^k)} \leq -\vareps \|p^k\|^2 \mbox{ and } \|p^k\|^2 \geq \vareps
\] (Algorithm~\ref{algo:2}, Line~\ref{line:1}). Hence, we have 
\[
\begin{split}
f(x^{k+1}) & = f(x^k+\alpha^kp^k) \\
& \leq f(x^k) + \alpha^k\sprod{\nabla f(x^k)}{p^k} + \frac{L}{2}\|\alpha^k p^k\|^2 \\
& \leq f(x^k) - \alpha^k \vareps\|p^k\|^2 + (\alpha^k)^2\frac{L}{2}\|p^k\|^2 \\
\end{split}
\]
If we set $\alpha^k = \frac{\vareps}{L}$, we get
\[
\begin{split}
f(x^{k+1}) \leq f(x^k) - \frac{\vareps^2}{2L}\|p^k\|^2 \leq f(x^k) - \frac{\varepsilon^3}{2L}.
\end{split}
\]
Hence, the objective function reduces at least by a positive constant that is bounded away from $0$ in each iteration.

In the second case, we have the following. For a function $f$ with $L$-Lipschitz continuous gradient and curvature pairs that satisfy the curvature condition $\sprod{y^k}{s^k}\geq \vareps \|y^k\|^2$, the smallest and largest eigenvalue of the Hessian approximation $B^k$ can, in general, be lower and upper bounded by two constants $c$ and $C$, see e.g.,~\cite{MokhtariR15}. Since we require this curvature condition to hold only for the index set $S^k$ that is active in iteration $k$ (see Algorithm~3 in the appendix), we can lower and upper bound the eigenvalues of the submatrix $B^k[S^k, S^k]$ of the Hessian approximation by $0<c$ and $C<\infty$. 
The quasi-Newton direction $d^k$ is computed by solving the equation
\[
-\nabla f(x^k)[S^k]=B^k[S^k, S^k] d^k[S^k].
\]
Multiplying both sides of this equation by $\left(d^k[S^k]\right)^\top$ gives
\[
\begin{split}
-\sprod{\nabla f(x^k)[S^k]}{d^k[S^k]} &=\left(d^k[S^k]\right)^\top B^k[S^k, S^k] d^k[S^k] \\ & \geq c\|d^k[S^k]\|^2.
\end{split}
\] 
Since $d^k[\bar{S^k}]=0$, this inequality can further be simplified to
\begin{equation}\label{eq:1}
\sprod{\nabla f(x^k)}{d^k}  \leq -c\|d^k\|^2.
\end{equation}

Since we are in the second case (Algorithm~\ref{algo:2}, Lines~\ref{line:2}--\ref{line:3}) we know that for all $i\in S^k$, if $d^k_i < 0$ and $x_i\leq l_i +\vareps$ it must hold that $\nabla f(x^k)_i < 0$,  otherwise $i\notin S^k$. In this case we have $d^k_i \cdot \nabla f(x^k)_i > 0$ and at the same time we set $p^k_i=0$. Hence, we have $d^k_i \cdot \nabla f(x^k)_i > p^k_i \cdot \nabla f(x^k)_i$. The case with $d^k_i > 0$ and $x_i\geq u_i +\vareps$ follows analogously. Hence, summing over all indices $i$, we can conclude $\sprod{\nabla f(x^k)}{d^k}\geq \sprod{\nabla f(x^k)}{p^k}$. Combining this inequality with Equation~\eqref{eq:1}, we get $-c\|d^k\|^2\geq\sprod{\nabla f(x^k)}{p^k}$.
Hence, we finally get
\begin{align*}
f(x^{k+1}) & = f(x^k+\alpha^kp^k) \\
& \leq f(x^k) + \alpha^k\sprod{\nabla f(x^k)}{p^k} + \frac{L}{2}\|\alpha^k p^k\|^2  \\
& \leq f(x^k) - \alpha^k c\|d^k\|^2 + (\alpha^k)^2 \frac{L}{2}\|d^k\|^2,
\end{align*}
because $\|p^k\| \leq \|d^k\|$.
Since all $x^k_i$ with $d^k_i\neq 0$ are at least $\vareps$ away from the boundary, we can pick $\alpha^k$ at least $\min_i \frac{\vareps}{|d^k_i|}=\frac{\vareps}{\|d^k\|_\infty}$.
If $\frac{c}{L}\leq \frac{\vareps}{\|d^k\|_\infty}$, we can set $\alpha^k=\frac{c}{L}$ and obtain the same result as in the first case. Otherwise, we set $\alpha^k=\frac{\vareps}{\|d^k\|_\infty}$ and obtain
\[
\begin{split}
f(x^{k+1}) & = f(x^k+\alpha^kp^k) \\ 
&\leq f(x^k) - \frac{\vareps}{\|d^k\|_\infty} c\|d^k\|^2 + \left(\frac{\vareps}{\|d^k\|_\infty}\right)^2\frac{L}{2}\|d^k\|^2 \\
& \leq f(x^k) - \frac{\vareps}{\|d^k\|_\infty} c\|d^k\|^2 + \frac{\vareps}{\|d^k\|_\infty}\frac{c}{L}\frac{L}{2}\|d^k\|^2 \\
&= f(x^k) - \frac{\vareps}{\|d^k\|_\infty}\frac{c}{2}\|d^k\|^2 \\
& \leq f(x^k) - \frac{\vareps c}{2}\|d^k\|,
\end{split}
\]
where the last line follows from $\|d^k\|_\infty \leq \|d^k\|$.
It remains to lower bound the Euclidean norm of $d^k$.  
We have $-\nabla f(x^k)[S^k]=B^k[S^k, S^k] d^k[S^k]$. Taking the squared norm on both sides, we have 
\[
\begin{split}
\norm{\nabla f(x^k)[S^k]}^2  & = \|B^k[S^k, S^k] d^k[S^k]\|^2 \\
& \hspace{-5ex} = \left(d^k[S^k]\right)^\top \left(B^k[S^k, S^k]\right)^\top B^k[S^k, S^k] d^k[S^k]\\
& \hspace{-5ex} \leq C^2 \|d^k[S^k]\|^2  = C^2 \|d^k\|^2
\end{split}
\]
since the eigenvalue of the submatrix $B^k[S^k, S^k]$ can be upper bounded by the constant $C$. 
As long as Algorithm~\ref{algo:1} has not converged, we know for the norm of the projected gradient that $\|\nabla f(x^k)[S^k]\| \geq \vareps$. Thus, we finally have
\[
f(x^{k+1}) \leq f(x^k) - \frac{\vareps c}{2}\|d^k\| \leq f(x^k) - \frac{\vareps^2 c}{2C}.
\]
Hence, also in the second case, the function value decreases by a positive constant in each iteration.

Thus, we make progress in each iteration by at least a small positive constant. Since $f$ is bounded from below, the algorithm will converge to a stationary point. Finally, note that $x^0$ is feasible. By construction and induction over $k$ it follows that $x^{k}$ is feasible for all $k$.
\end{proof}

\begin{corollary} If $f$ satisfies the assumptions of Theorem~\ref{thm:1} and is convex, then Algorithm~\ref{algo:1} converges to a global optimal point.
\end{corollary}

Algorithm~\ref{algo:1} uses the \texttt{projectDirection} subroutine (Algorithm~\ref{algo:2}). It becomes apparent from the proof that one can skip the projection branch (Line~\ref{line:1}) and instead always follow the modified quasi-Newton direction and still obtain convergence guarantees. However, here, we use a projection as it was similarly suggested in~\cite{MoralesN11} which often reduces the number of iterations in the original L-BFGS-B algorithm~\cite{ByrdLNZ95}. 

\section{Complete Framework} \label{sec:framework}
In the previous section we have described our approach for solving optimization problems with box constraints that can be efficiently run on a GPU. We extend this approach to also handle arbitrary constraints by using an augmented Lagrangian approach. This extension allows to solve constrained optimization problems of the form
\begin{equation} \label{eq:constrained}
\min_{x}  f(x) \:\,\st\:\, h(x) = 0,\, g(x) \leq 0, \textrm{ and } l \leq x \leq u,
\end{equation}
where $x\in\setR^n$, $f\colon \setR^n \to \setR$, $h\colon\setR^n\to\setR^m$, $g\colon\setR^n\to\setR^p$ are differentiable functions, and the equality and inequality constraints are understood component-wise.

The augmented Lagrangian of Problem~\eqref{eq:constrained} is the
following function
\begin{equation}
\label{eq:augLag}
\begin{aligned}
  L (x, \lambda, \mu, \rho) = f(x) & + \frac{\rho}{2}
  \norm{h(x)+\lambda /\rho}^2 \\
  & + \frac{\rho}{2} \norm{\left(g(x)
    + \mu /\rho\right)_+}^2,
\end{aligned}
\end{equation}
where $\lambda\in\setR^m$ and $\mu\in\setR_{\geq 0}^p$ are Lagrange
multipliers, \mbox{$\rho >0$} is a constant, and $(v)_+$ denotes $\max\{v, 0\}$. The Lagrange multipliers are also referred to as dual variables.

The augmented Lagrangian Algorithm (see Algorithm~4 in the appendix) runs in iterations. In each iteration it minimizes the augmented Lagrangian function, Eq.~\eqref{eq:augLag}, subject to the box constraints using Algorithm~\ref{algo:1} and updates the Lagrange multipliers $\lambda$ and $\mu$. If Problem~\eqref{eq:constrained} is convex, the augmented Lagrangian algorithm returns a global optimal solution. Otherwise, it returns a local optimum~\cite{Bertsekas99}. 

\noindent
\begin{minipage}{0.5\columnwidth}
We integrated our solver with the modeling framework GENO presented in~\cite{LaueMG2019} that allows to specify the optimization problem in a natural, easy-to-read modeling language, see for instance the example to the right. Based on the matrix and tensor calculus methods presented in~\cite{LaueMG2018,LaueMG2020}, \hfill the
\end{minipage}
\hfill
\begin{minipage}{0.4\columnwidth}
  \begin{Verbatim}[frame=single]
parameters
  Matrix A
  Vector b
variables
  Vector x 
min 
  norm2(A*x-b)
st
  sum(x) == 1
  x >= 0
  \end{Verbatim}
\vfill
\end{minipage}
framework then generates Python
code that computes function values and gradients of the objective function and the constraints.
The code maps all linear expressions to NumPy statements. Since any NumPy-compatible library can be used within the generated code this allows us to replace NumPy by CuPy to run the solvers on the GPU. We extended the modeling language and the Python code generator to our needs here. An interface can be found at \href{https://www.geno-project.org}{\texttt{https://www.geno-project.org}}.

Our framework and solvers are solely written in Python, which makes them easily portable as long as NumPy-compatible interfaces are available. Here, we use the CuPy library~\cite{cupy} in order to run the generated solvers on the GPU. 
The code for the solver is available the github repository \href{https://www.github.com/slaue/genosolver}{\texttt{https://www.github.com/slaue/genosolver}}.

\ignore{
\section{Limitations}
The solvers generated by our approach can solve constrained optimization problems. For convex problems, they return the global optimal solution. In the non-convex case however, only a local minimum is returned. Also, they do not make use of any special structure of the problem. Hence, solvers that are  specifically designed to solve the specified problem can be more efficient. Since the whole problem needs to fit into memory, a limiting factor can be the GPU RAM which is usually less than the CPU RAM.
}

 \begin{figure*}[t]
  \centering
  \includegraphics[width=0.24\textwidth]{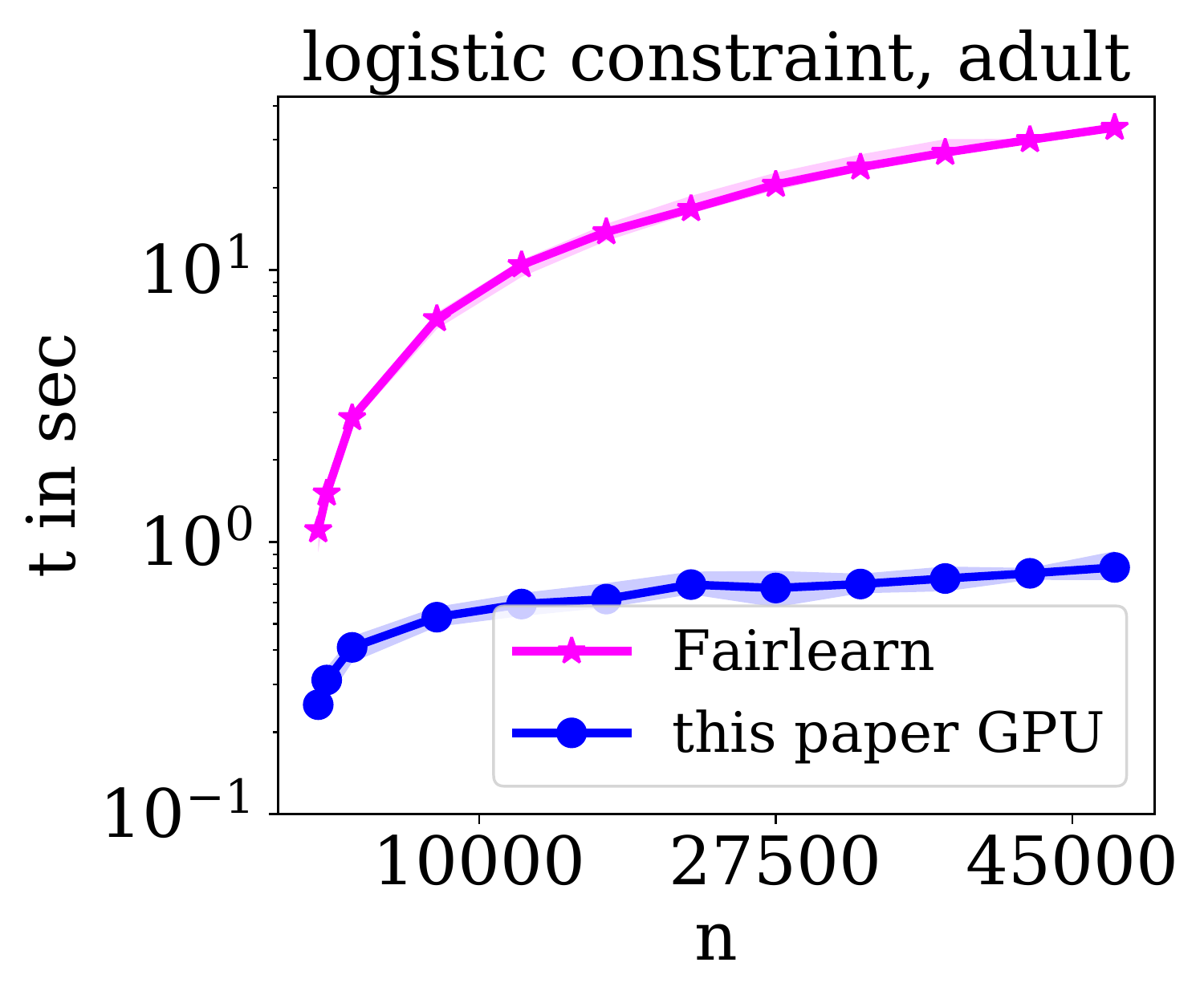}
  \includegraphics[width=0.24\textwidth]{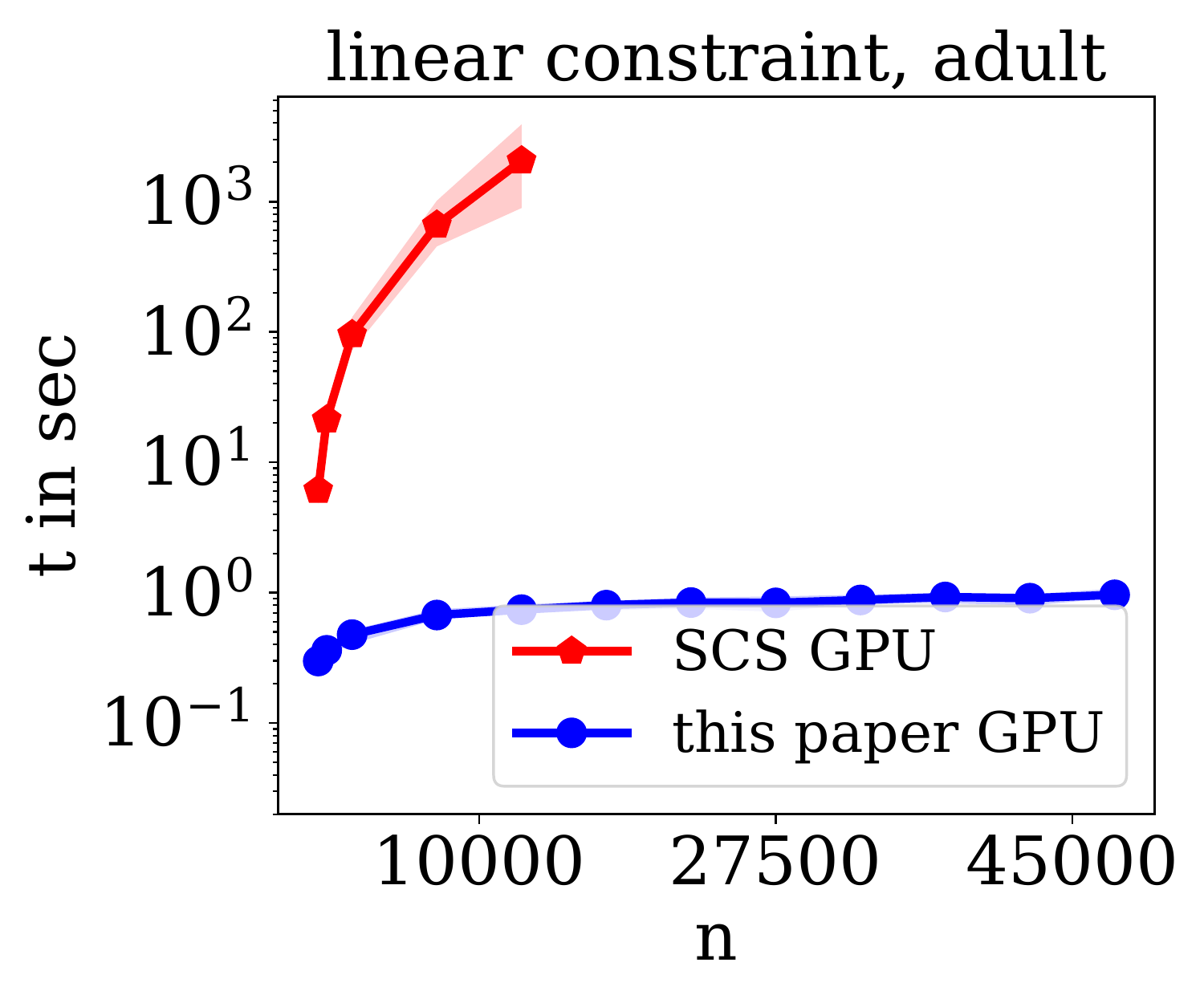}
  \includegraphics[width=0.24\textwidth]{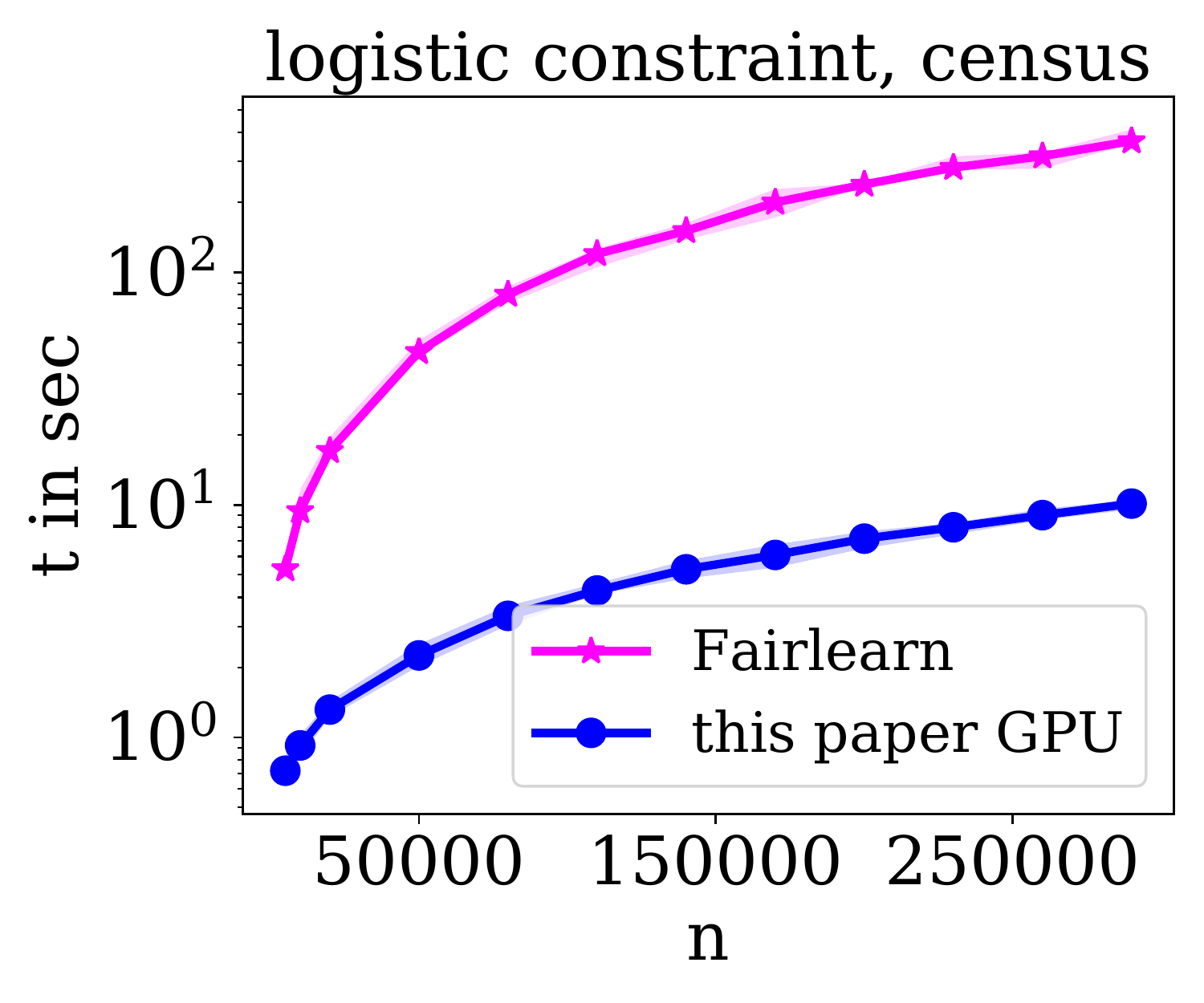}
  \includegraphics[width=0.24\textwidth]{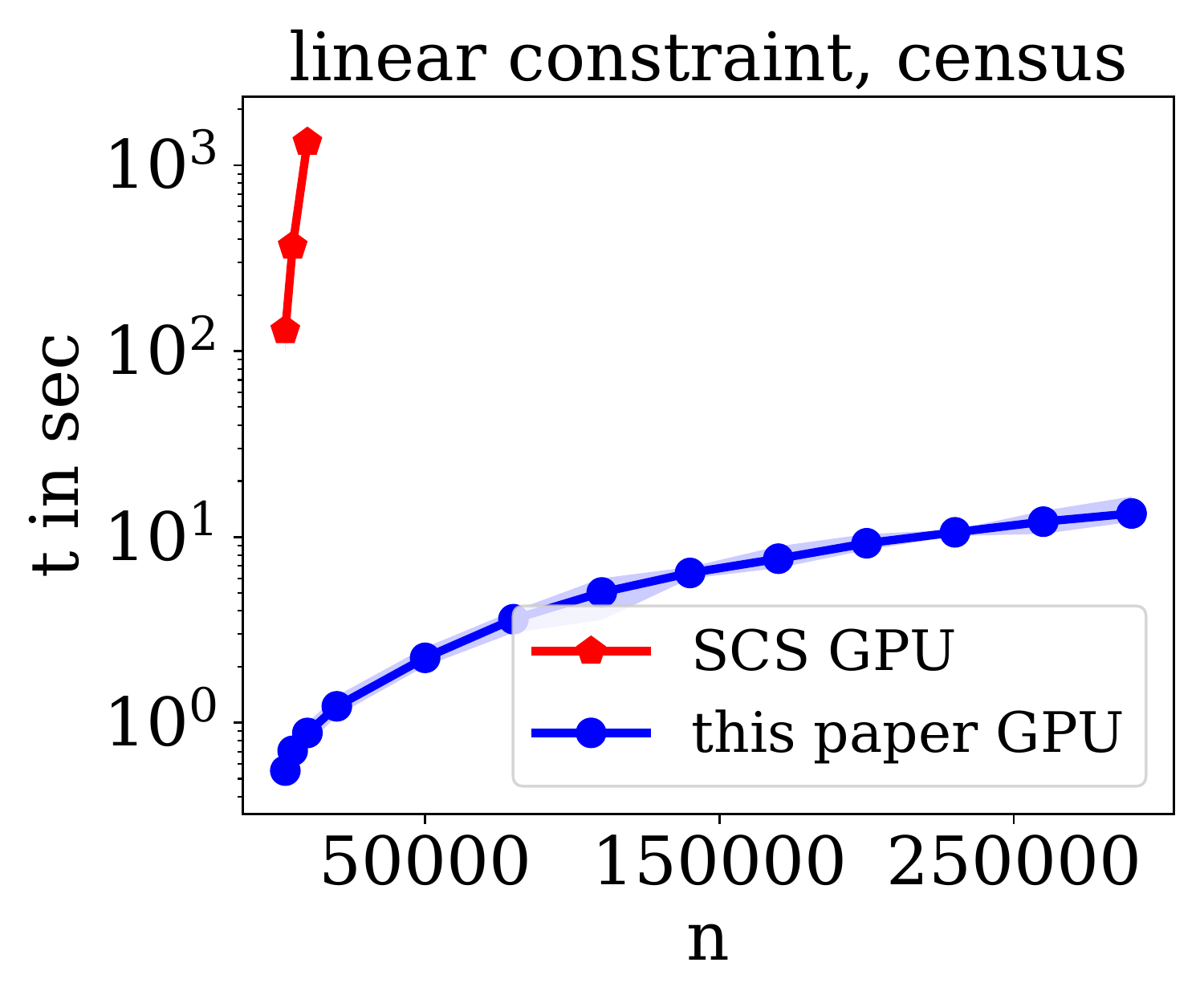}
  \caption{Running times for the logistic regression problem with fairness constraints. The two plots on the left show the running times for the adult data set and the two plots on the right for the census data set. For each data set, one plot shows the running times when the logistic loss is used in the fairness constraint and one plot when the linear loss is used.}
  \label{fig:fairlearn}
\end{figure*}

\section{Experiments} \label{sec:experiments}
The purpose of the following experiments is to show the efficiency of our approach on a set of different classical, that is, non-deep, machine learning problems. For that purpose, we selected a number of well-known classical problems that are given as constrained optimization problems, where the optimization variables are either vectors or matrices. All these problems can also be solved on CPUs. 
In the supplemental material, we provide results that show that this GPU version of the GENO framework significantly outperforms the previous, efficient multi-core CPU version of the GENO framework~\cite{LaueMG2019}.
Here, we compare our framework on the GPU to CVXPY paired with the cuOSQP and SCS solvers. CVXPY has a similar easy-to-use interface and also allows to solve general constrained optimization problems. Note however, that CVXPY is restricted to convex problems and cuOSQP to convex quadratic problems. It was shown that cuOSQP outperforms its CPU version OSQP by about a factor of ten~\cite{SchubigerBL20}. Our experiments confirm this observation.

To the best of our knowledge, pairing CVXPY with \mbox{cuOSQP} or SCS \emph{are the only two approaches comparable to ours.} 
Another framework that has been released recently that can solve convex, constrained optimization problems on the GPU is cvxpylayers~\cite{AgrawalABBDK19}. However, its focus is on making the solution of the optimization problem differentiable with respect to the input parameters for which it needs to solve a slightly larger problem. Internally, it uses CVXPY combined with the SCS solver in GPU-mode. Hence, this framework is slower than the original combination of CVXPY and SCS. 

In all our experiments we made sure that the solvers that were generated by our framework always computed a solution that was \emph{better} than the solution computed by the competitors in terms of objective function value and constraint satisfaction. All experiments were run on a machine equipped with an Intel i9-10980XE 18-core processor running Ubuntu 20.04.1 LTS with 128 GB of RAM, and a Quadro RTX 4000 GPU that has 8 GB of GDDR6 SDRAM and 2304 CUDA cores. Our framework took always less than 10 milliseconds for generating a solver from its mathematical description. 
\nocite{GiesenL16,FunkeLS17,FunkeLS16}

 \begin{figure*}[t]
  \centering
  \includegraphics[width=0.32\textwidth]{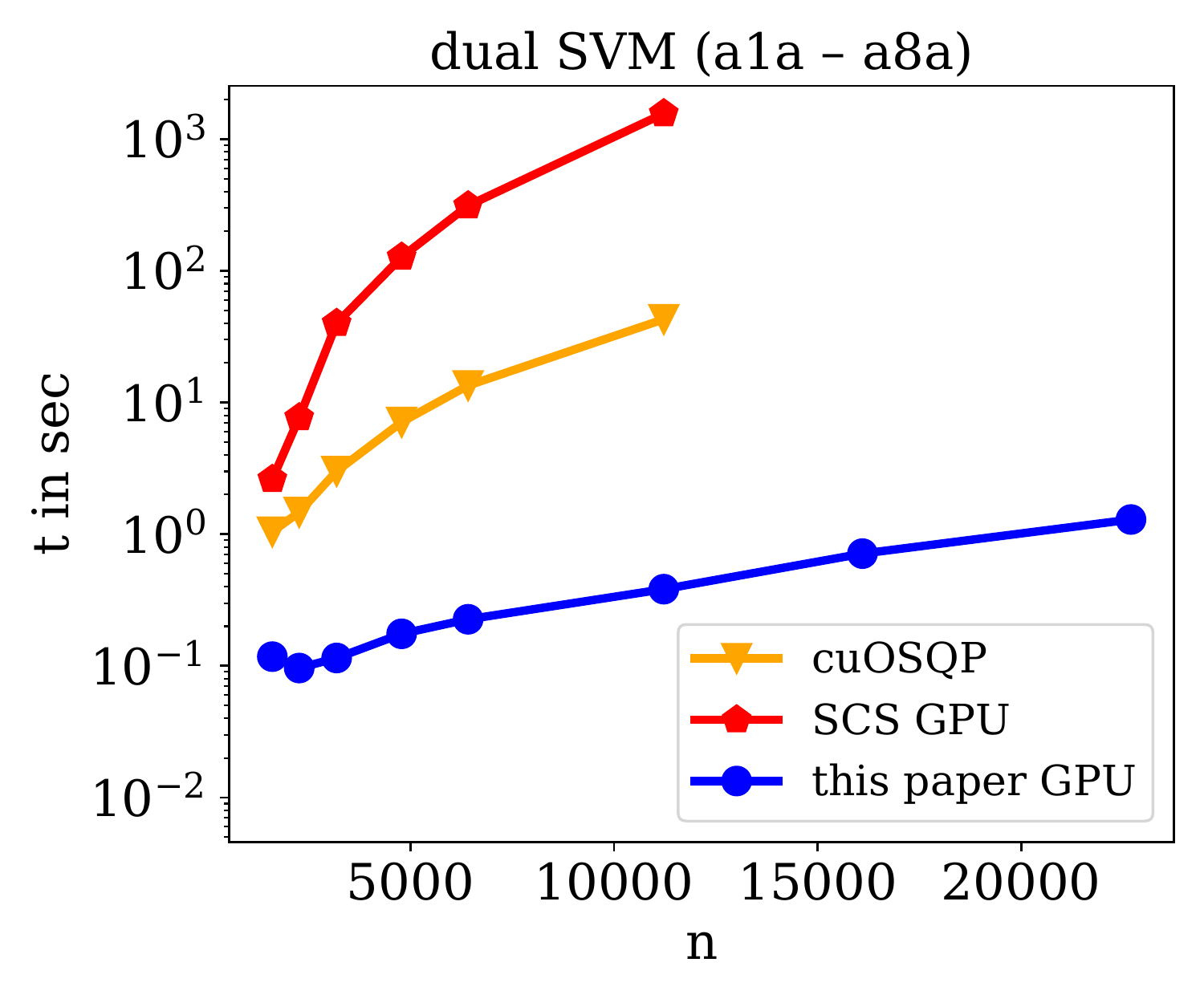}
  \includegraphics[width=0.32\textwidth]{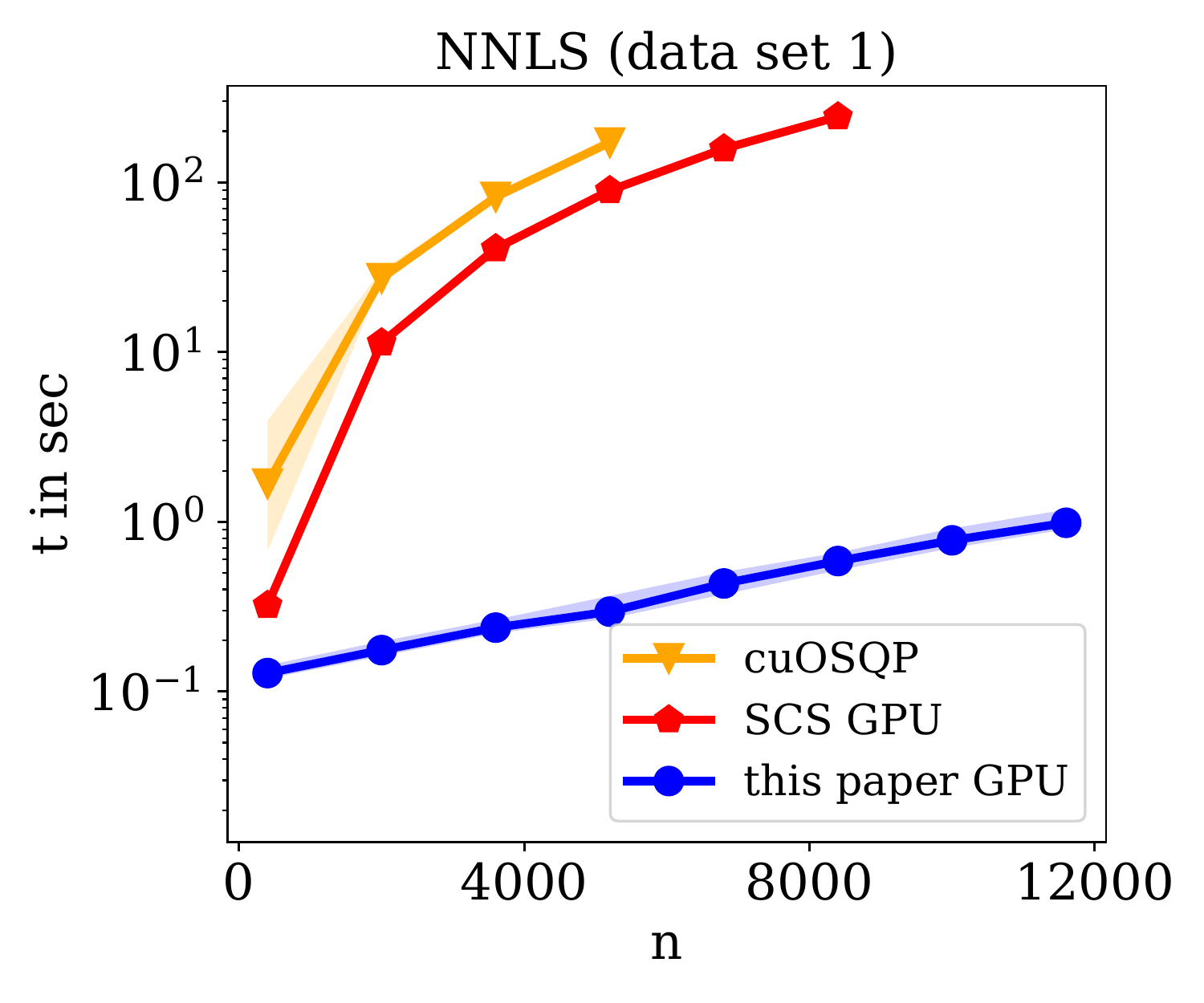}
  \includegraphics[width=0.32\textwidth]{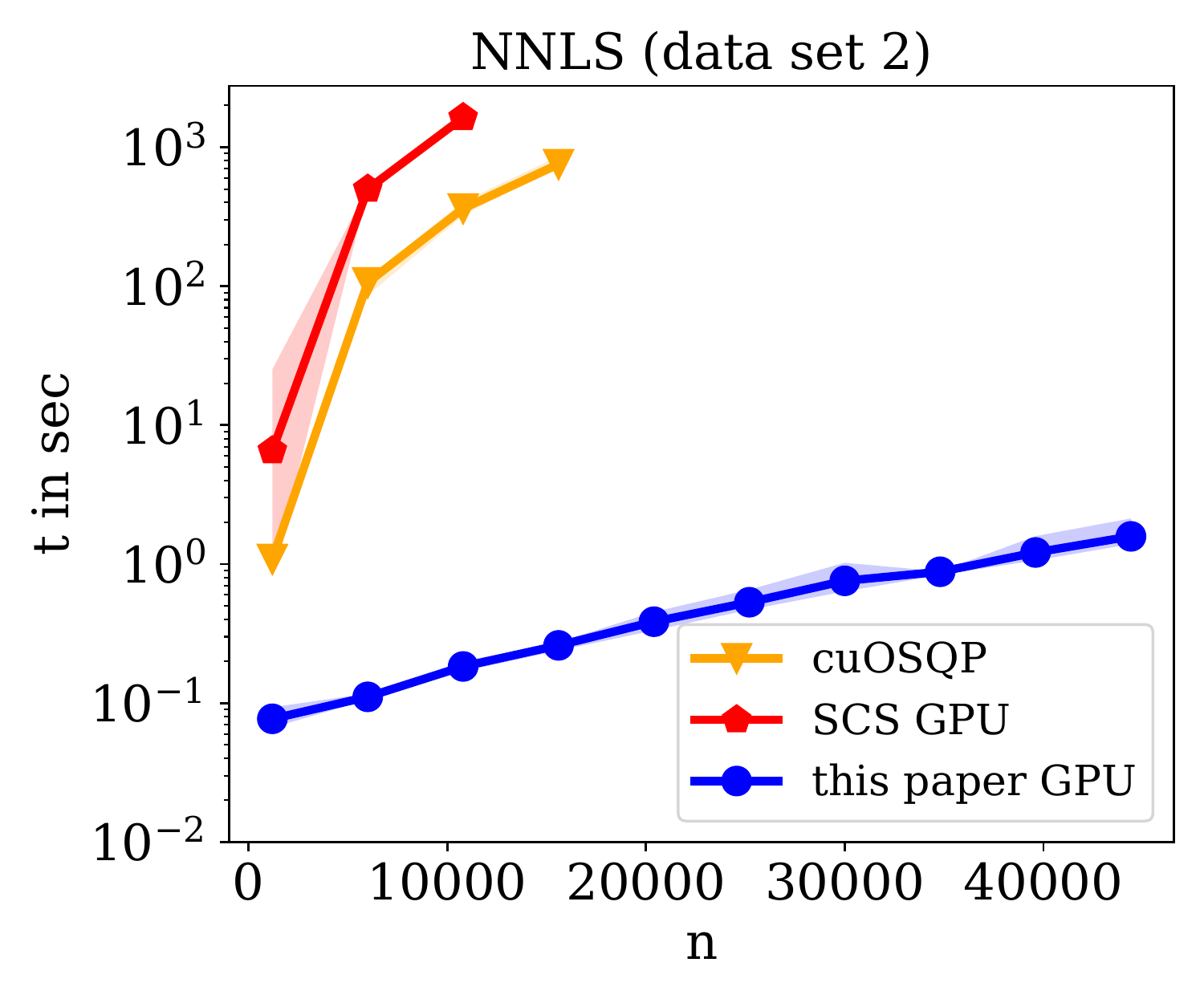}
  \caption{
The plot on the left shows the running times for the SVM problem on the adult data set. The plot in the middle and the plot on the right show the running times for the non-negative least squares problem when run on the first and second data set, respectively.}
  \label{fig:SVM_NNLS}
\end{figure*}

\begin{table*}[h!]
  \centering
  \begin{tabular}{l*{8}{r}}
    \toprule
    \multirow{2}{*}{Solver} & &&&{Data sets} \\
    \cmidrule{2-8}
    & cod-rna & covtype & ijcnn1 & mushrooms & phishing & a9a & w8a \\
    \midrule
    this paper GPU & 0.6 & 0.1 & 1.8 & 0.1 & 1.6 & 0.3 & 1.4 \\
    cuOSQP & 55.7 & failed & 206.5 & 22.0 & 163.5 & 32.8 & 36.1 \\
    SCS GPU & 2342.0 & 31.3 & N/A & 7994.9 & N/A & 1094.0 & 1227.1 \\
    \bottomrule
  \end{tabular}
  \caption{Running times in seconds for the dual SVM problem. All data sets were subsampled to $10,000$ data points due time and memory requirements of the cuOSQP and SCS solver. N/A indicates that the solver did not finish within $10,000$ seconds.}
\label{tab:SVM}
\end{table*}

\subsection{Fairness in Machine Learning}
In classical machine learning approaches, one usually minimizes a regularized empirical risk in order to make correct predictions on unseen data. However, due to various causes, e.g., bias in the data, it can happen that some group of the input data is favored over another group. 
Such favors can be mitigated by the introduction of  constraints that explicitly prevent them. This is the goal of fairness in machine learning which has gained considerable attention over the last few years. Here, we consider fairness for binary classification.

There are a number of different fairness definitions~\cite{AgarwalBD0W18,barocas-hardt-narayanan,DoniniOBSP18}, see the Fairlearn project~\cite{bird2020fairlearn, fairlearn} for an introduction and overview. Here, we follow the exposition and formulation in~\cite{DoniniOBSP18}. Let $D=\{(x^1, y^1), \ldots, (x^m, y^m)\}$ be a labeled data set and $A$ and $B$ be two groups, i.e., subsets of the data set. Then, one seeks to find a classifier that is statistically independent of the group membership $A$ and $B$. 
Depending on the type of groups $A$ and $B$, respectively, different types of fairness constraints are obtained. Since statistical independence is defined with respect to the true data distribution, which is typically unknown, one replaces the expectation over the true distribution by the empirical risk.
Hence, one solves the following constrained optimization problem
\[
\min_{f}\: \widehat L_D(f) + \lambda \cdot r(f)\\
  \quad \st\quad \widehat L_A (f) = \widehat L_B (f),
\]
where $f\colon X\to\setR$ is a function or model, $l\colon \setR\times Y\to \setR$  is a loss function, $\widehat L_D(f) = \frac{1}{|D|} \sum_{(x^i, y^i) \in D} l(f(x^i), y^i)$ is the empirical risk of $f$ over the data set $D$, and $r(\cdot )$ is the regularizer. 

Ideally, one would like to use the same loss function for the risk minimization $\widehat L_D(f)$ as in the fairness constraint $\widehat L_A(f) = \widehat L_B (f)$. The logistic loss is often used for classification. However, when the logistic loss is used in the fairness constraint, the problem becomes non-convex, even for a linear classifier. Using our framework, we can still solve this problem. However, we cannot compare its performance to cuOSQP or SCS since they only allow to solve convex problems. Thus, we compare it to the exponentiated gradient approach~\cite{AgarwalBD0W18} paired with the Liblinear solver~\cite{FanCHWL08}. 
Note, that this approach does not run on the GPU. However, to provide a better global picture, we still include it here. 
Only when the loss function in the fairness constraint is linear as in~\cite{DoniniOBSP18, ZemelWSPD13}, the problem becomes convex. We also consider this case and compare it to SCS. Note, the problem cannot be solved by cuOSQP since it contains exponential cones. 

We used the same setup, the same data sets, and the same preprocessing as described in the Fairlearn package~\cite{bird2020fairlearn}. We used the adult data set ($48,842$ data points with $120$ features) and the census-income data set ($299,285$ data points with $400$ features) each with `female' and `male' as the two subgroups. For each experiment, we sampled $m$ data points from the full data set. Figure~\ref{fig:fairlearn} shows the running times. Our framework provides similar results in terms of quality as the exponentiated gradient approach, when the logistic loss is used in the fairness constraint, and it is orders of magnitude faster than SCS on the GPU, when the linear loss is used in the fairness constraint.

\subsection{Dual SVM}

Support vector machines (SVMs) are a classical yet still relevant classification method. When combined with a kernel, they are usually solved in the dual problem formulation, which reads as
\[
\min_{a}\, \frac{1}{2} (a\odot y)^\top K (a\odot y) - \|a\|_1
  \, \st\, y^\top a = 0, 0\leq a\leq c,
\]
where $K\in\setR^{n\times n}$ is a positive semidefinite kernel matrix, $y\in \{-1, +1\}^n$ are the corresponding binary labels, $a\in\setR^n$ are the dual variables, $\odot$ is the element-wise multiplication, and $c\in\setR_+$ is the regularization parameter. 

We used all data sets from the LibSVM data sets website~\cite{libsvmData} that had more than $8000$ data points with fewer than $1000$ features such that a kernel approach is reasonable. We applied a standard Gaussian kernel with bandwidth parameter $\gamma=1$ and regularization parameter $c=1$. Table~\ref{tab:SVM} shows the running times for the data sets when subsampled to $10,000$ data points. Figure~\ref{fig:SVM_NNLS} shows the running times for an increasing number of data points based on the original subsampled adult data set~\cite{libsvmData}. It can be seen that our approach outperforms \mbox{cuOSQP} as well as SCS by several orders of magnitude. The cuOSQP solver ran out of memory for problems with more than $10,000$ data points. While there is a specialized solver for solving these SVM problems on the GPU~\cite{thundersvm18}, the focus here is on general purpose frameworks.

 \begin{figure*}[t]
  \centering
  \includegraphics[width=0.32\textwidth]{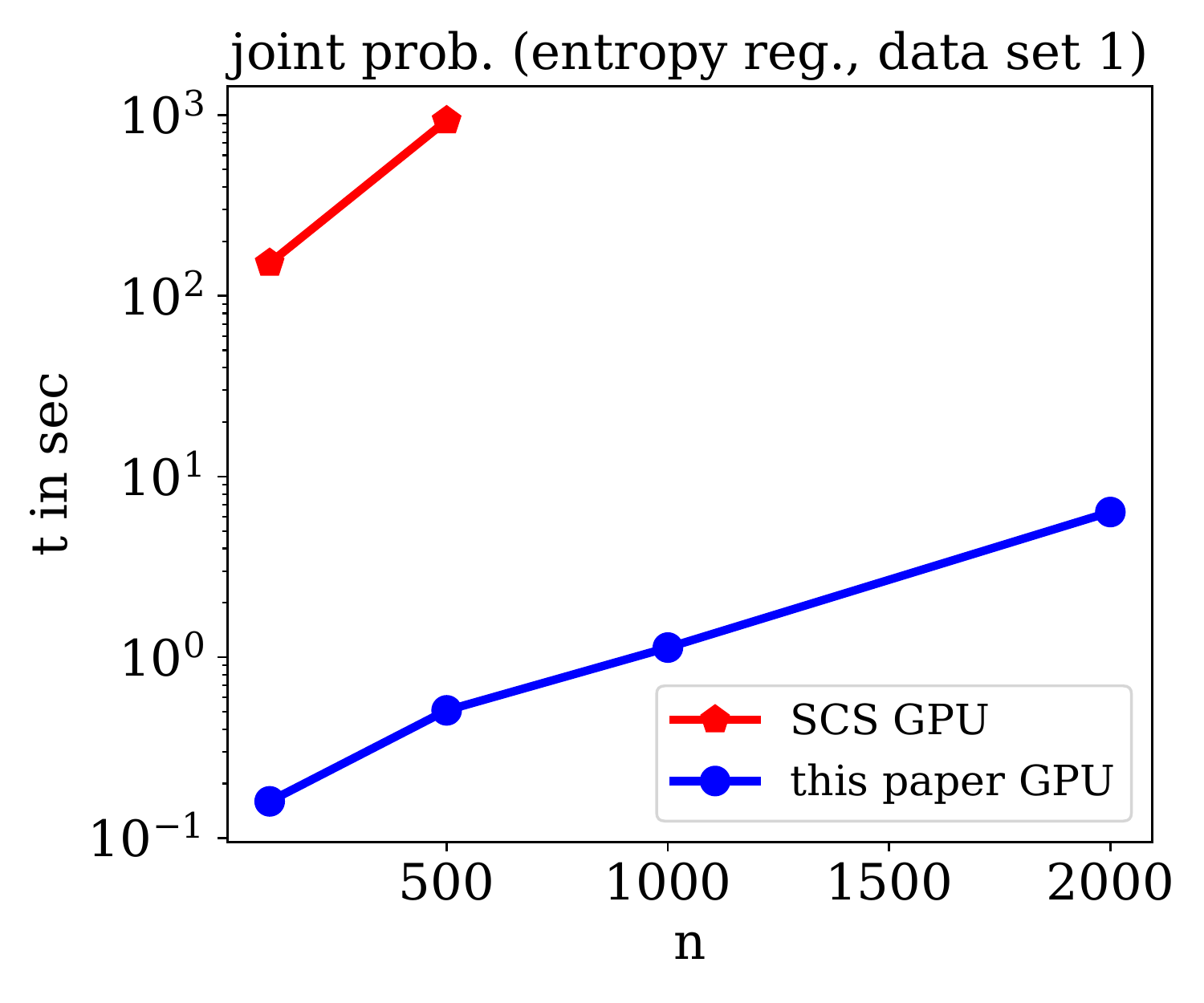}
  \includegraphics[width=0.32\textwidth]{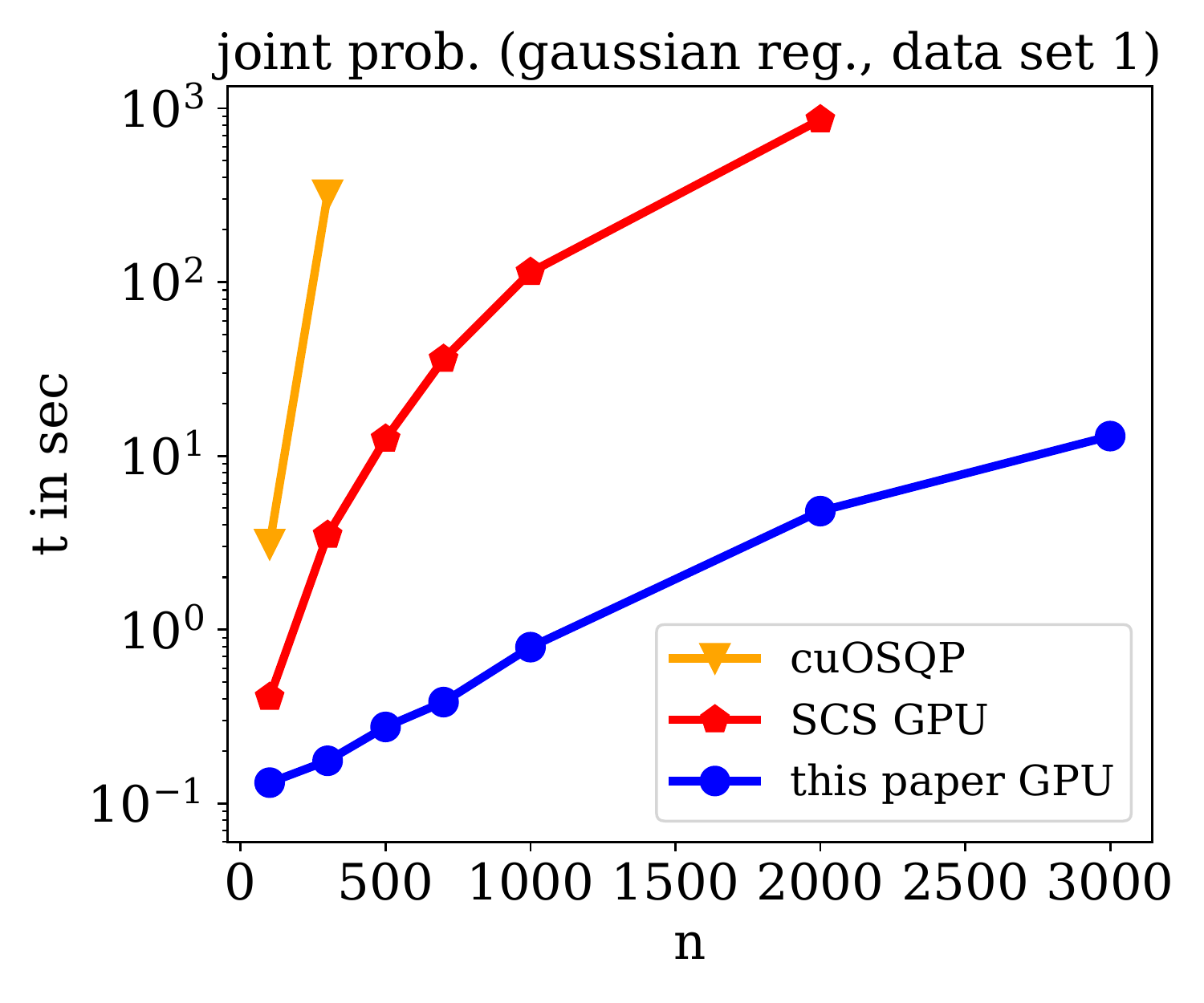}
  \includegraphics[width=0.32\textwidth]{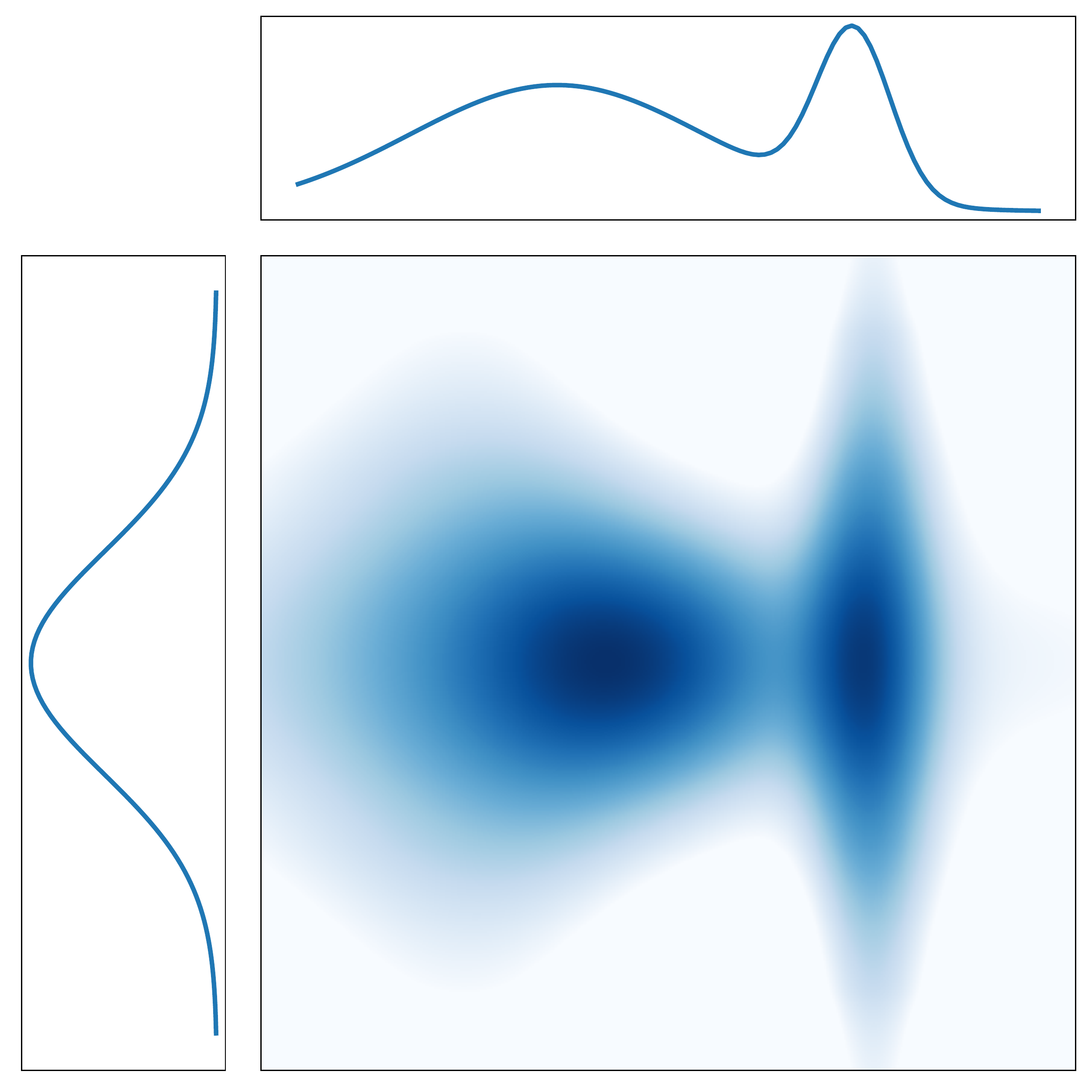}
  \caption{Running times for computing the joint probability distribution from two marginal distributions. The left plot shows the running times when the entropy prior is used and the plot in the middle when a Gaussian prior is used. The right figure visualizes the probabilities.}
  \label{fig:jointProb}
\end{figure*}

\subsection{Non-negative Least Squares}
Non-negative least squares is an extension of the least squares regression problem that requires the output to be non-negative. See~\cite{Slawski13} for an overview on the non-negative least squares problem. It is given as the following optimization problem
\[
\min_{x}\, \norm{Ax-b}_2^2 \quad\st\quad x\geq 0,
\]
where $A\in\setR^{m\times n}$ is the design matrix and
$b\in\setR^m$ the response vector. We ran two sets of
experiments, similarly to the comparisons in~\cite{Slawski13}, where
it was shown that different algorithms behave quite differently on
these problems. For experiment (i), we generated a random data matrix
$A\in\setR^{2000\times 6000}$, where the entries of $A$ were sampled
uniformly at random from the unit interval and a sparse vector
$x\in\setR^{6000}$ with non-zero entries sampled from the standard Gaussian distribution and a sparsity of $0.01$. The response variables 
were then generated as $y = \sqrt{0.003}\cdot Ax + 0.003 \cdot z$,
where $z\sim {\cal{N}}(0, 1)$. For experiment (ii),
$A\in\setR^{6000\times 3000}$ was drawn from a Gaussian distribution
and $x$ had a sparsity of $0.1$. The response variable was generated as
$y=\sqrt{1/6000}\cdot Ax+0.003\cdot z$, where $z\sim {\cal{N}}(0,
1)$. The differences between the two experiments are:
(1) The Gram matrix $A^\top A$ is singular in experiment (i) and
regular in experiment (ii), (2) the design matrix $A$ has isotropic
rows in experiment (ii) but not in experiment (i), and
(3) $x$ is significantly sparser in (i) than in (ii). 
To evaluate the runtime behavior for increasing problem size, we scaled the problem sizes to $A\in\setR^{2000t \times 6000t}$ in the first experiment and to $A\in\setR^{6000t\times 3000t}$ in the second experiment for a parameter $t\in[0, 6]$. For each problem instance we performed ten runs and report the average running time along with the standard deviation in Figure~\ref{fig:SVM_NNLS}. We stopped including SCS and cuOSQP into the experiments once their running time exceeded $1000$ seconds. It can be seen that SCS is faster than cuOSQP on the first set of experiments and slower than cuOSQP on the second set. However, our approach outperforms SCS and cuOSQP by several orders of magnitude in both sets of experiments.

\subsection{Joint Probability Distribution}
Given two discrete probability distributions $u\in\setR^m$ and $v\in\setR^n$, we are interested in their joint probability distribution $P\in\setR^{m\times n}$. 
This problem has been studied intensively before, see, e.g.,~\cite{Cuturi13, FrognerP19, MuzellecNPN17}. With additional knowledge, it can be reduced to a regularized optimal transport problem. Many different regularizers have been used, for instance, an entropy, a Gaussian, or more generally, a Tsallis regularizer. The corresponding optimization problem is the following constrained optimization problem over positive matrices
\[
\min_{P}\: \sprod{M}{P} + \lambda \cdot r(P) \quad \st\quad P\cdot \vecOne = u,\, P^\top \cdot \vecOne = v,\, 0\leq P,
\]
where $M\in\setR^{m\times n}$ is the cost matrix, $r(.)$ is the regularizer, $\vecOne$ is the all-ones vector, and $\lambda\in\setR_+$ is the regularization parameter.
In our experiments we used the entropy and the Gaussian regularizer that are both special cases of the Tsallis regularizer. In the special case that the regularizer is the entropy, $m=n$, and the cost matrix $M$ is a metric, \cite{Cuturi13} showed that the problem can be solved using Sinkhorn's algorithm~\cite{Sinkhorn67}. Similar results are known for other special cases~\cite{JanatiMPC20}. However, in general  Sinkhorn's algorithm cannot be used as it is the case in the present experiments, since the cost matrix is not a metric.

Here, we used synthetic data sets since the real-world data sets that are usually used for this task are very small, typically $m, n \leq 20$. We created two sets of synthetic data sets. For the first set of data sets, we let a Gaussian and a mixture of two Gaussians be the marginals, see Figure~\ref{fig:jointProb}. Then, we discretized both distributions to obtain the marginal vectors $u$ and $v$. In this case, we set $m=n$. Hence, when $n=1000$, the corresponding optimization problem has $10^6$ optimization variables with lower bound constraints and $2000$ equality constraints. The cost matrix $M$ was fixed to be the discretized version $uu^\top$ of a two-dimensional Gaussian, and the regularization parameter was set as $\lambda=\frac{1}{2}$. We ran two sets of experiments on this data set, one where $r(.)$ is the entropy regularizer and another one with the Gaussian regularizer. Figure~\ref{fig:jointProb} shows the running times for both experiments for varying problem sizes. It can be seen that our approach outperforms cuOSQP and SCS by several orders of magnitude. The cuOSQP solver ran out of memory already on very small problems. The second data set was created as in~\cite{FrognerP19}. On this data set, the speedup over cuOSQP and SCS is even more pronounced. Detailed results can be found in the appendix. 
 
\section{Conclusion}
We presented an approach for solving constrained optimization problems on the GPU efficiently and included this approach into the GENO framework.
The framework allows to specify a constrained optimization problem in an easy-to-read modeling language  and then generates solvers in portable Python code that outperform competing state-of-the-art approaches on the GPU by several orders of magnitude. Using GPUs also for classical, that is, non-deep, machine learning becomes increasingly important as hardware vendors 
started to combine CPUs and GPUs on a single chip like Apple's M1 chip. 


\section*{Acknowledgments}
This work was supported by the German Science Foundation (DFG) grant (GI-711/5-1) within the priority program (SPP~1736) \emph{Algorithms for Big Data} and by the Carl Zeiss Foundation within the project \emph{A Virtual Werkstatt for Digitization in the Sciences}.

\bibliography{gpugeno}

\appendix
\onecolumn
\newpage
\newgeometry{left=35mm,right=35mm,top=35mm,bottom=30mm}
\section{Appendix}
In this supplementary material, we provide more details on the original L-BFGS-B algorithm, missing algorithmic details of our approach, and more experiments. The focus of the main paper is to provide an efficient algorithmic framework for solving constrained optimization problems on the GPU. However, to provide a better overall picture, we also provide comparisons for a CPU version of our approach. We compare the CPU version of our approach to its GPU version and also to the original GENO framework~\cite{LaueMG2019}, which specifically targets CPUs. Comparing running times across CPUs and GPUs is not always fair, since computations on different hardware cannot be directly compared. In order to ensure approximately fair comparisons, we used a CPU and a GPU that cost about the same. Namely, we used an  Intel i9-10980XE 18-core CPU and a Quadro RTX 4000 GPU with 2304 CUDA cores.

\section{Comparison to Original L-BFGS-B on the GPU}

In this section, we further discuss the algorithmic shortcomings of the original L-BFGS-B when ported to the GPU and provide further technical details, as already mentioned in the introduction of the main paper.

The original L-BFGS-B algorithm, like any other quasi-Newton method, approximates the function to be minimized by a quadratic model. In each iteration, the algorithm first computes the set of fixed variables, i.e., the optimization variables that are on their bounds. This is achieved as follows: The quadratic model is minimized along the path which starts at the current iterate and points into the direction of the gradient. There are two possibilities: Either the minimum is attained at the current path segment or the path first hits the boundary of the feasible region. In the first case, the minimum along the projected gradient path is found. This point is called the Cauchy point. In the second case, the path is projected back onto the feasible region and the minimization along the projected path continues. Pseudo-code and explicit formulas can be found in~\cite{ByrdLNZ95}. The main observation is that computing the minimum of the quadratic approximation along a ray only needs a few scalar operations. Whenever the path hits the boundary, the ray changes direction and the minimum along this new ray needs to be computed. This loop is repeated until the minimum is found. The total number of iterations of this loop is usually roughly of the order of the number of variables. This is an inherent sequential part that cannot benefit from parallelization. While this loop is executed on the CPU as well as on the GPU only on one core, its running time increases drastically on the GPU since its cores are much weaker. This problem is reflected in experiments and it was the main reason for designing our new algorithm. 

Consider for instance the non-negative least squares problem, as described in the experiments section of the main paper. The absolute error of $10^{-10}$ was similar for both solvers and the number of iterations (between 30-40) was identical for both solvers on each data set. As can be seen from the results in Table~\ref{tab:nnls_0} the original L-BFGS-B algorithm even slows down on the GPU. The reason is the bottleneck of the Cauchy point computation. Our approach does not suffer from the problem and parallelizes nicely on the GPU.

\begin{table}[h!]
\begin{center}
  \small
  \begin{tabular}{l*{9}{r}}
    \toprule
   problem size  & \multicolumn{2}{c}{6000} & \multicolumn{2}{c}{8000} & \multicolumn{2}{c}{10,000} & \multicolumn{2}{c}{12,000} \\
    & total time & CP  time & total time & CP time & total time & CP time & total time & CP time\\
    \midrule
    L-BFGS-B CPU & 1.5 & 0.3 & 2.6 & 0.4 & 3.6 & 0.5 & 4.9 & 0.6\\
    L-BFGS-B GPU & 2.8 & 2.5 & 3.5 & 3.1 & 4.1 & 3.6 & 5.2 & 4.6\\
    this paper GPU & 0.3 &     & 0.4 &      & 0.5 &      & 0.8 &\\
    \bottomrule
  \end{tabular}
  \caption{Detailed running time comparison: depicted is the total running time in seconds as well as the time in seconds spent for the Cauchy point computation (CP) for the non-negative least squares problem for varying problem sizes. Larger problem sizes do not fit into the GPU RAM anymore.}
  \label{tab:nnls_0}
  \end{center}
\end{table}

\newpage
\section{Missing Algorithmic Details}
Here, we provide details for the two-loop algorithm and the augmented Lagrangian algorithm. 

\subsection{Modified Two-loop Recursion}
In general, the purpose of the two-loop recursion algorithm~\cite{NocedalW99} is to solve the quasi-Newton equation for computing the new search direction $d^k$, i.e.,
\[
   -\nabla f(x^k)=B^k d^k,
\]
where $B^k$ is the Hessian approximation in iteration $k$. Since some of the variables are fixed in iteration $k$, the equation needs to be solved only for the subset $S^k$ of 
free variables, i.e.,
\[
   -\nabla f(x^k)[S^k]=B^k[S^k, S^k] d^k[S^k].
\]

The new search direction that is computed by the L-BFGS update rule needs to be a descent direction. In order to satisfy this constraint, the Hessian approximation $B^k$ needs to be positive definite. In general, if each correction pair $(y^i, s^i)$ satisfies the curvature condition $\sprod{y^i}{s^i}\geq \vareps \|y^i\|^2$, then $B^k$ has a smallest eigenvalue that can be bounded by a positive constant $c$, see~\cite{MokhtariR15}. Even when the objective function is convex and the curvature condition is satisfied for all $i\leq k$, it can happen that the curvature condition is violated on the subspace of the free variables with indices in $S^k$, i.e., it can even happen that $\sprod{y^i[S^k]}{s^i[S^k]}< 0$. Using this correction pair for computing the next search direction, does not provide a descent direction. Hence, in order for the method to work, the corresponding curvature condition needs to be checked for all stored curvature pairs and the current index set $S^k$. Algorithm~\ref{algo:3} incorporates this strategy. 

As mentioned in the main paper, the strong Wolfe conditions in the line search assure that the curvature condition is satisfied for the whole correction pair $(y^i, s^i)$. However, since this does not imply the curvature condition for the reduced correction pair $(y^i[S^k], s^i[S^k])$, it is not necessary to satisfy the strong Wolfe conditions. Instead, the weaker Armijo conditions are sufficient in the line search. The Armijo conditions are often satisfied after fewer steps in the line search.

\setcounter{algorithm}{2}
\begin{algorithm}[h!]
   \caption{Modified Two-loop Recursion}
   \label{algo:3}
\begin{algorithmic}[1]
   \STATEx {\hspace{-0.5cm} \bfseries Input:} gradient $\nabla f(x^k)$, index set $S^k$
   \vspace{0.2cm}

  \STATE $q=\nabla f(x^k)[S^k]$   
   \FOR{$i=k-1, \ldots , k-m$}
   \STATE $\rho^i=\sprod{s^i[S^k]}{y^i[S^k]}$
   \IF{$\rho^i > \vareps \|y^i\|^2$}
   \STATE $\alpha^i = \frac{1}{\rho^i} \cdot \sprod{s^i[S^k]}{q}$
   \STATE $q = q- \alpha^i \cdot y^i[S^k]$
   \ENDIF
   \ENDFOR

   \IF{$\rho^{k-1} > \vareps \|y^{k-1}\|^2$}
   \STATE $q = \frac{\rho^{k-1}}{\|y^{k-1}\|^2} \cdot q$
   \ENDIF
  
   \FOR{$i=k-m, \ldots , k-1$}
   \IF{$\rho^i > \vareps \|y^i\|^2$}
   \STATE $\beta = \frac{1}{\rho^i} \cdot \sprod{y^i[S^k]}{q}$
   \STATE $q = q + (\alpha^i-\beta) \cdot s^i[S^k]$
   \ENDIF
   \ENDFOR
   \RETURN q
\end{algorithmic}
\end{algorithm}

\subsection{Augmented Lagrangian Algorithm}
The presented algorithm can solve optimization problems with box constraints, i.e.,  upper and lower bounds on the variables. In order to solve general constrained optimization problems, i.e., 
\begin{equation}\label{eq:constrained}
\begin{array}{rl}  \displaystyle
  \min_{x} &  f(x) \\
  \st &  h(x) = 0 \\
  	& g(x) \leq 0 \\
  	& l \leq x \leq u,
\end{array}
\end{equation}
we use the augmented Lagrangian algorithm. It reduces the constrained optimization problem to a sequence of box-constrained optimization problems by incorporating the constraints into the augmented Lagrangian of the problem, i.e., 
\begin{equation}
\label{eq:augLag}
  L (x, \lambda, \mu, \rho) = f(x) + \frac{\rho}{2}
  \norm{h(x)+\lambda /\rho}^2 + \frac{\rho}{2} \norm{\left(g(x)
    + \mu /\rho\right)_+}^2,
\end{equation}
where $\lambda\in\setR^m$ and $\mu\in\setR_{\geq 0}^p$ are Lagrange
multipliers, $\rho >0$ is a constant, and $(v)_+$ denotes $\max\{v, 0\}$. 

The augmented Lagrangian algorithm is shown in Algorithm~\ref{algo:4}. It runs in iterations and minimizes the augmented Lagrangian function~\eqref{eq:augLag} in each iteration using Algorithm~1. Then, it updates the Lagrangian multipliers $\lambda$ and $\mu$. If the infinity norm of the constraint violation is not halved in an iteration, then $\rho$ is multiplied by a factor of 2. Convergence of the augmented Lagrangian algorithm was shown in~\cite{Bertsekas99, Birgin14}.

\begin{algorithm}[h!]
  \caption{Augmented Lagrangian Algorithm}
  \label{algo:4}
  \begin{algorithmic}[1]
    \STATE {\bfseries input:} constrained optimization Problem~\eqref{eq:constrained}
    \STATE {\bfseries output:} approximate solution $x\in\setR^{n},
    \lambda\in\setR^{p}, \mu\in\setR_{\geq 0}^{m}$ 
    \STATE initialize $x^0 = 0$, $\lambda^0 = 0$, $\mu^0 = 0$, and $\rho=1$
    \REPEAT
    \STATE  $x^{k+1} :=\quad \argmin_{l\leq x\leq u}\, L(x, \lambda^k, \mu^k, \rho)$ \label{algo:x}
    \STATE $\lambda^{k+1} :=\quad  \lambda^k + \rho h(x^{k+1})$ \label{algo:lambda}
    \STATE $\mu^{k+1} :=\quad  \left(\mu^k + \rho g(x^{k+1})\right)_+$ \label{algo:mu}
    \STATE update $\rho$ \label{algo:rho}
    \UNTIL{convergence}
    \RETURN $x^k, \lambda^k, \mu^k$
  \end{algorithmic}
\end{algorithm}

\clearpage
\section{Experiments}
Here, we present the comparisons of our framework including its CPU version and the GENO framework~\cite{LaueMG2019} for CPUs. We also include the running times for the joint probability experiment on the second data set that was excluded from the main paper due to space constraints. It can be seen that our framework on the GPU outperforms the highly efficient GENO framework by a large margin. Since the solvers generated by our framework are written entirely in Python, they can also run on the CPU by mapping all linear algebra expressions to NumPy instead of CuPy. This allows to run our framework also on the CPU. The corresponding entry in the tables is `this paper CPU'. In all experiments, it can be seen  that the GPU version of our framework outperforms all other frameworks once the size of the data set is reasonably large. For small data sets, the full capabilities of the GPU cannot be exploited, and hence, it does not pay off to run them on the GPU. Note again, the generated solvers of our approach were run until they obtained a \emph{smaller} objective function value and constraint violation than the competing approaches. The absolute errors were usually between $10^{-3}$ and $10^{-5}$.

\subsection{Fairness in Machine Learning}

 \begin{figure}[h]
  \centering
  \includegraphics[width=0.24\textwidth]{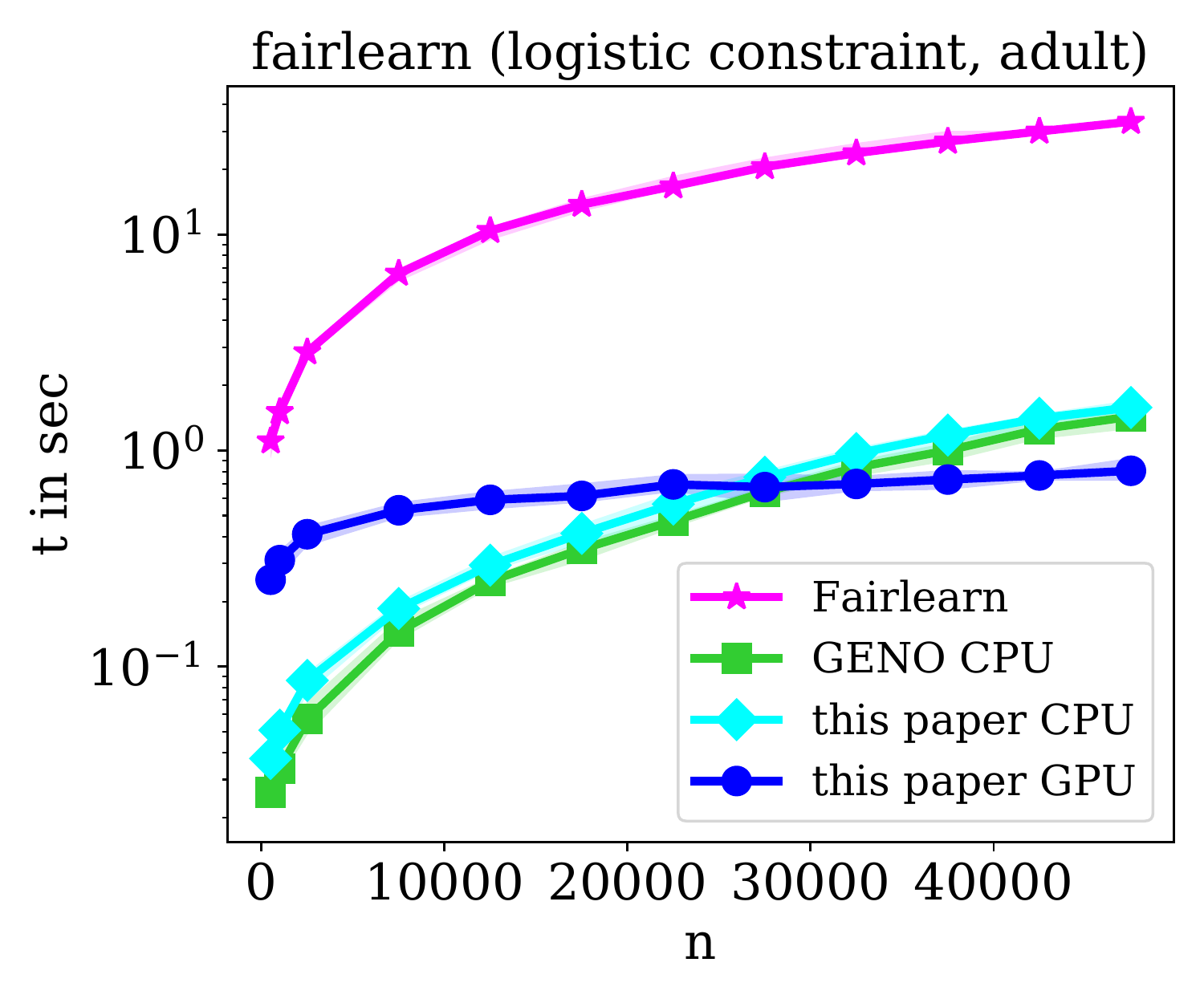}
  \includegraphics[width=0.24\textwidth]{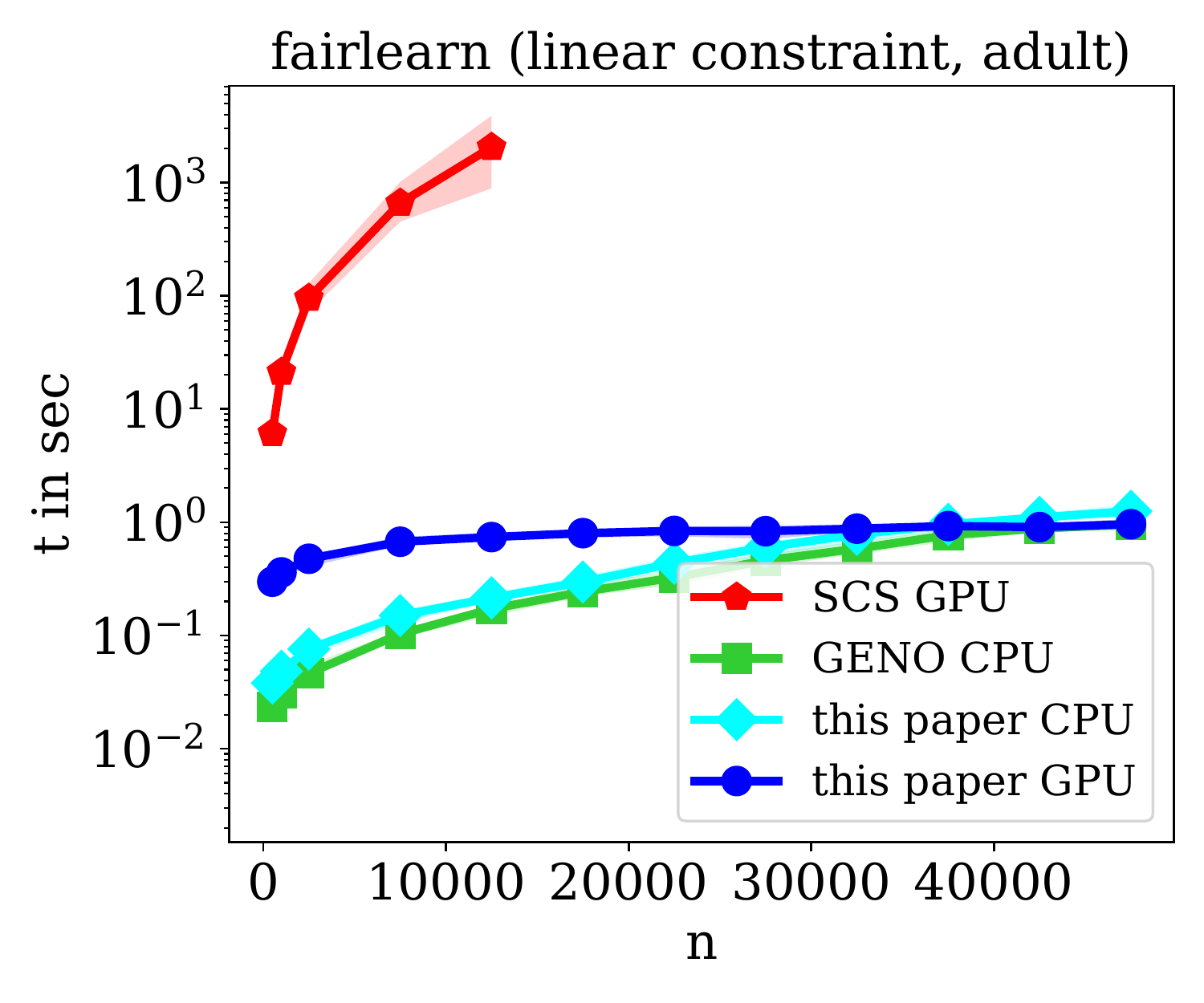}
  \includegraphics[width=0.24\textwidth]{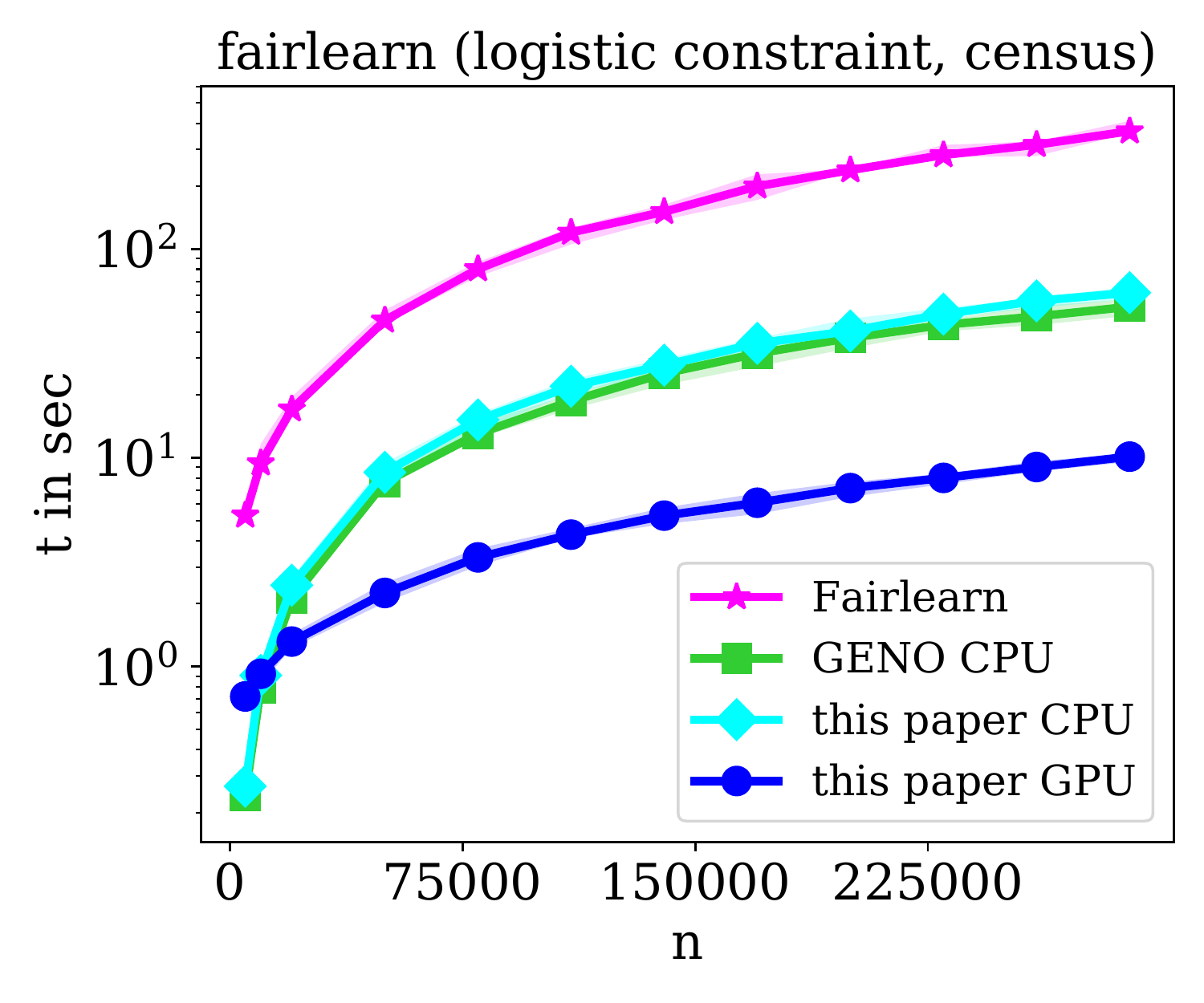}
  \includegraphics[width=0.24\textwidth]{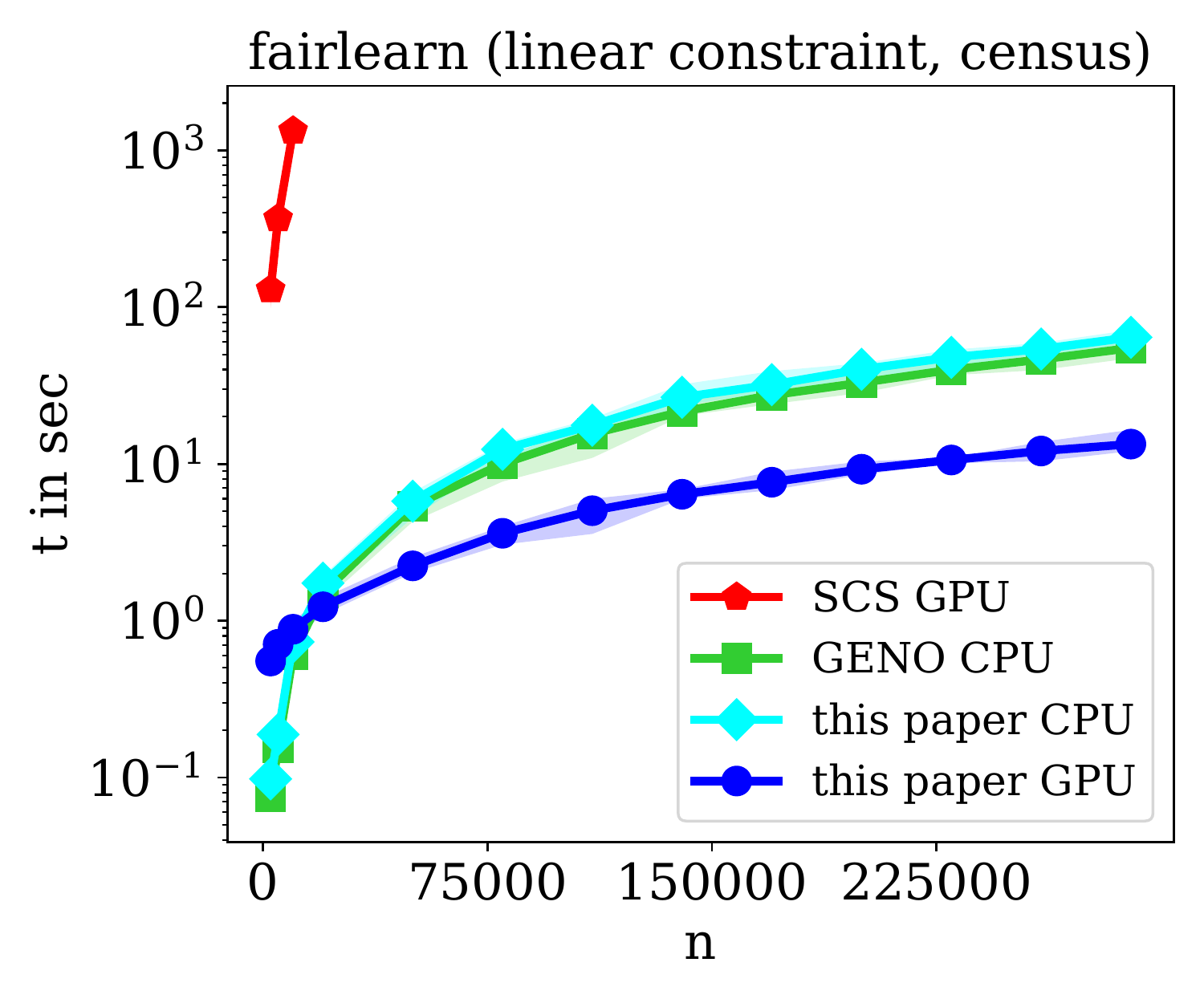}
  \caption{Running times for the logistic regression problem with fairness constraints. The two plots on the left show the running times for the adult data set and the two plots on the right for the census data set. For each data set, one plot shows the running times for using the logistic loss in the fairness constraint and one plot for using the linear loss.}
  \label{fig:fairlearn}
\end{figure}

\begin{table}[h!]
  \begin{center}
{\small
  \begin{tabular}{l*{7}{r}}
    \toprule
    \multirow{2}{*}{Solver} & \multicolumn{6}{c}{fairlearn (logistic constraint, adult)} \\
    \cmidrule{2-7}
    & 500 & 2500 & 12500 & 22500 & 32500 & 42500 \\
    \midrule
    this paper GPU	& $0.3\pm 0.0$ &  $0.4\pm 0.0$ & $0.6\pm 0.0$ & $0.7\pm 0.0$ & $0.7\pm 0.0$ & $0.8\pm 0.0$ \\
    this paper CPU	& $0.0\pm 0.0$  & $0.1\pm 0.0$ & $0.3\pm 0.0$ & $0.6\pm 0.0$ & $1.0\pm 0.1$ & $1.4\pm 0.1$ \\
    GENO CPU		& $0.0\pm 0.0$ &  $0.1\pm 0.0$ & $0.2\pm 0.0$ & $0.5\pm 0.0$ &  $0.8\pm 0.0$ & $1.3\pm 0.1$ \\
    Fairlearn		& $1.1\pm 0.1$ & $2.9\pm 0.0$ & $10.4\pm 0.7$  & $16.7\pm 1.0$ & $23.8\pm 1.4$ & $30.0\pm 0.2$ \\
    \bottomrule
  \end{tabular}}
  \caption{Running times in seconds for fairlearn (logistic constraint, adult).}
  \end{center}
\end{table}

\begin{table}[h!]
  \begin{center}
  \small
  \begin{tabular}{l*{7}{r}}
    \toprule
    \multirow{2}{*}{Solver} & \multicolumn{6}{c}{fairlearn (linear constraint, adult)} \\
    \cmidrule{2-7}
    & 500 & 2500 & 12500 & 22500 & 32500 & 42500 \\
    \midrule
    this paper GPU & $0.3\pm 0.0$ & $0.5\pm 0.0$   & $0.7\pm 0.0$       & $0.8\pm 0.1$ & $0.9\pm 0.1$ & $0.9\pm 0.1$ \\
    this paper CPU & $0.0\pm 0.0$ & $0.1\pm 0.0$   & $0.2\pm 0.0$       & $0.4\pm 0.0$ & $0.8\pm 0.1$ & $1.1\pm 0.1$ \\
    GENO CPU       & $0.0\pm 0.0$ & $0.0\pm 0.0$   & $0.2\pm 0.0$       & $0.3\pm 0.0$ & $0.6\pm 0.0$ & $0.9\pm 0.0$ \\
    SCS GPU        & $6.1\pm 0.8$ & $95.5\pm 17.9$ & $2056.8\pm 971.6$  & N/A & N/A & N/A  \\
    \bottomrule
  \end{tabular}
  \caption{Running times in seconds for fairlearn (linear constraint, adult).}
  \end{center}
\end{table}

\begin{table}[h!]
\begin{center}
  \small
  \begin{tabular}{l*{7}{r}}
    \toprule
    \multirow{2}{*}{Solver} & \multicolumn{6}{c}{fairlearn (logistic constraint, census)} \\
    \cmidrule{2-7}
                   & 10000        & 50000         & 110000         & 170000          & 230000          & 290000 \\
    \midrule                                                                                           
    this paper GPU & $0.9\pm 0.1$ & $2.3\pm 0.1$  & $4.3\pm 0.1$   & $6.1\pm 0.4$    & $8.0\pm 0.3$    & $10.1\pm 0.3$ \\
    this paper CPU & $0.9\pm 0.1$ & $8.5\pm 0.5$  & $22.0\pm 0.9$  & $35.3\pm 1.3$   & $48.7\pm 1.6$   & $61.7\pm 3.3$ \\
    GENO CPU       & $0.8\pm 0.0$ & $7.6\pm 0.4$  & $18.7\pm 0.9$  & $31.6\pm 2.0$   & $43.3\pm 2.0$   & $52.9\pm 2.8$ \\
    Fairlearn      & $9.4\pm 1.0$ & $45.5\pm 2.8$ & $120.0\pm 7.5$ & $199.7\pm 17.6$ & $281.9\pm 11.9$ & $366.6\pm 16.0$ \\
    \bottomrule
  \end{tabular}
  \caption{Running times in seconds for fairlearn (logistic constraint, census).}
  \end{center}
\end{table}

\begin{table}[h!]
\begin{center}
\small
  \begin{tabular}{l*{7}{r}}
    \toprule
    \multirow{2}{*}{Solver} & \multicolumn{6}{c}{fairlearn (linear constraint, census)} \\
    \cmidrule{2-7}
					& 2500            & 10000             & 50000         & 170000        & 230000        & 290000 \\
    \midrule                                                                                              
    this paper GPU 	& $0.6\pm 0.0$    & $0.9\pm 0.1$      & $2.2\pm 0.2$  & $7.6\pm 0.7$  & $10.6\pm 0.3$ & $13.4\pm 1.4$ \\
    this paper CPU 	& $0.1\pm 0.0$    & $0.7\pm 0.0$      & $5.8\pm 0.6$  & $32.0\pm 2.9$ & $47.8\pm 4.5$ & $64.2\pm 5.7$ \\
    GENO CPU       	& $0.1\pm 0.0$ 	  & $0.6\pm 0.0$      & $5.4\pm 0.7$  & $27.4\pm 3.0$ & $39.9\pm 3.1$ & $54.8\pm 7.3$ \\
    SCS GPU        	& $129.3\pm 18.6$ & $1331.3\pm 146.7$ & N/A & N/A & N/A & N/A  \\
    \bottomrule
  \end{tabular}
  \caption{Running times in seconds for fairlearn (linear constraint, census).}
  \end{center}
\end{table}

\newpage

\subsection{Support Vector Machines}

 \begin{figure}[h]
  \centering
  \includegraphics[width=0.32\textwidth]{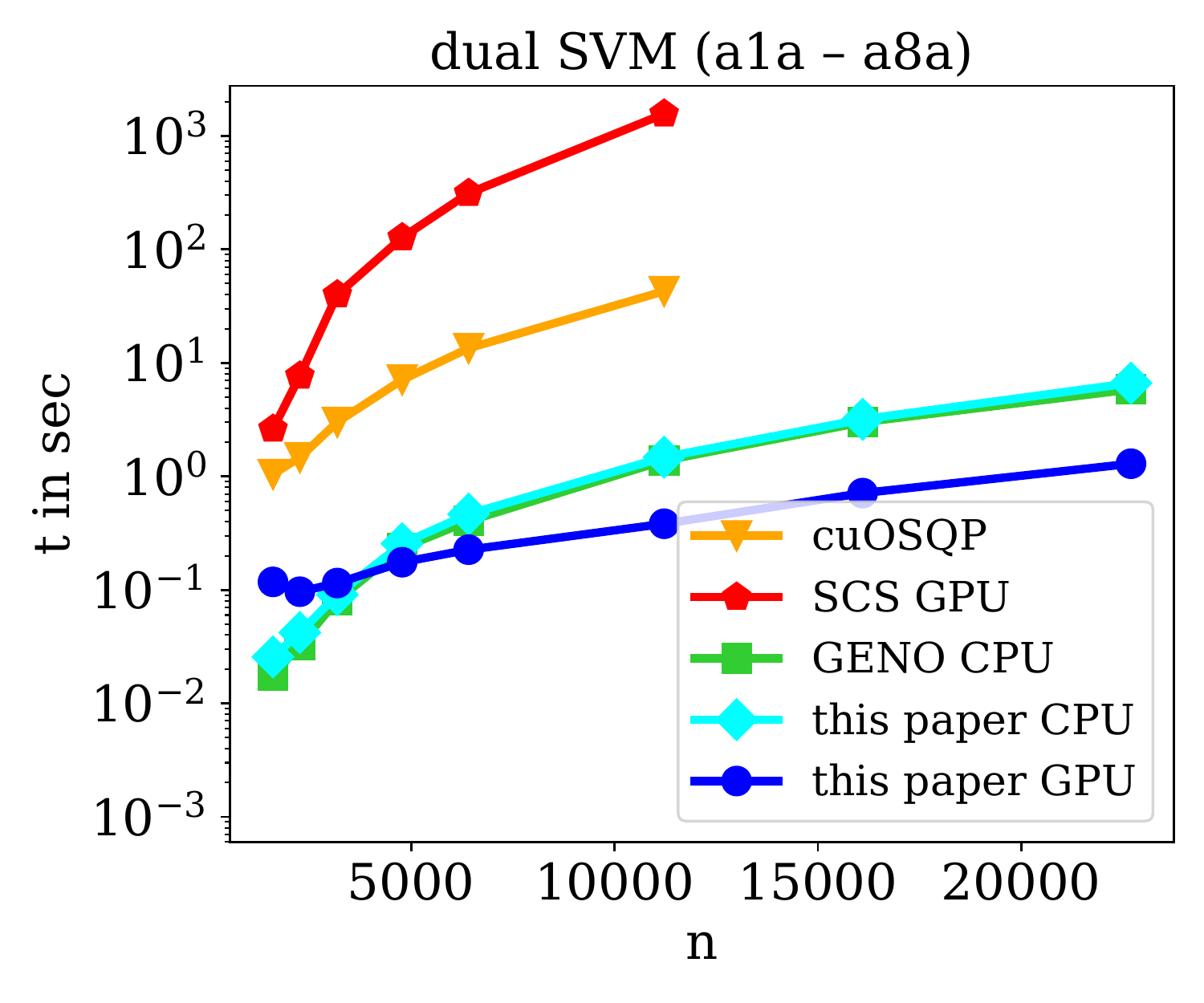}
  \includegraphics[width=0.32\textwidth]{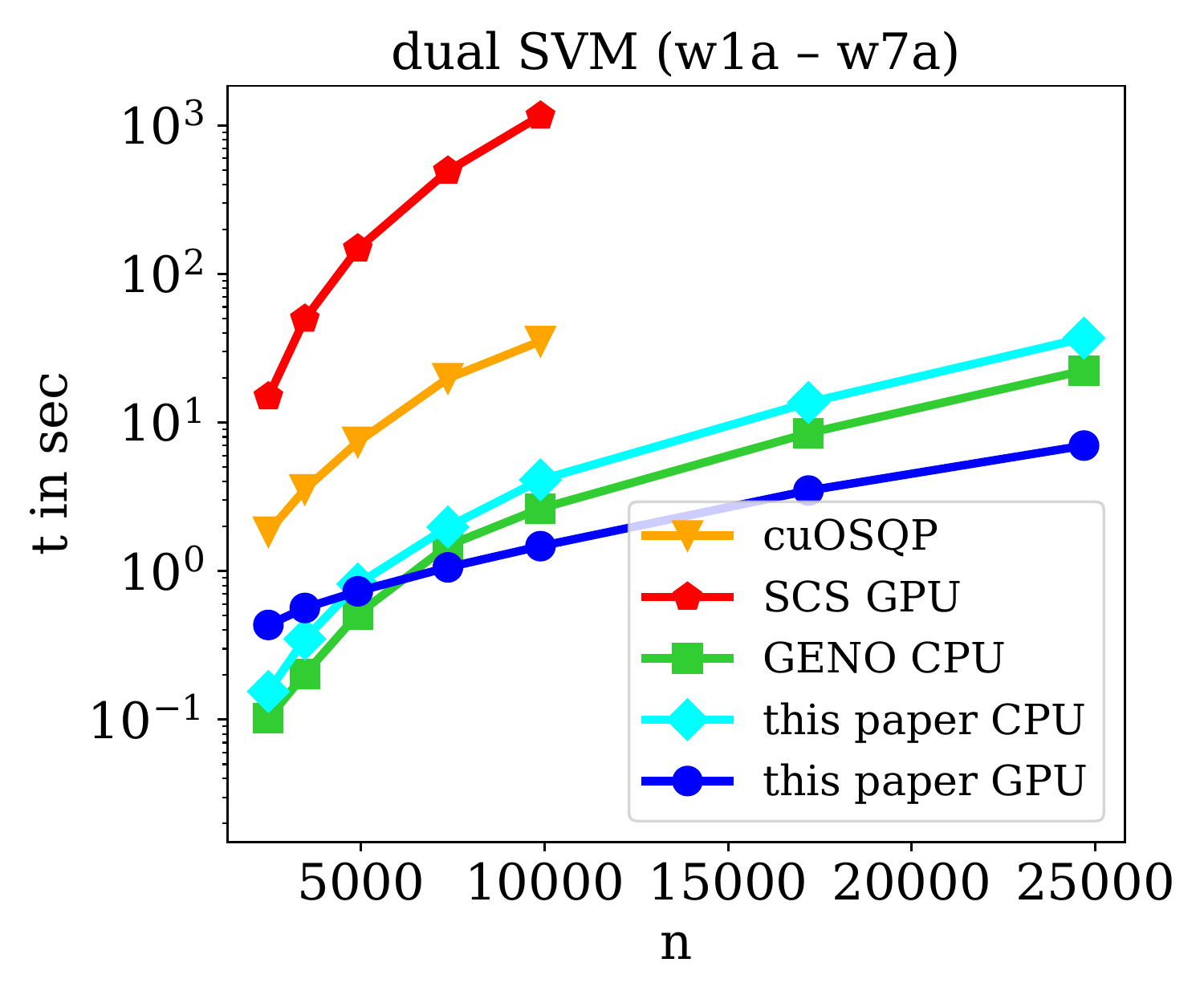}
  \caption{
The plot on the left shows the running times for the SVM problem on the adult data set for an increasing number of data points. The plot on the right shows the web data set for an increasing number of data points.}
  \label{fig:SVM}
\end{figure}

\begin{table}[h!]
\small
  \begin{center}
  \label{tab:SVM}
  \begin{tabular}{l*{8}{r}}
    \toprule
    \multirow{2}{*}{Solver} & \multicolumn{7}{c}{Data sets} \\
    \cmidrule{2-8}
    & cod-rna & covtype & ijcnn1 & mushrooms & phishing & a9a & w8a \\
    \midrule
    this paper GPU & 0.7 & 0.1 & 1.8 & 0.1 & 1.7 & 0.3 & 1.5 \\
    this paper CPU & 2.0 & 0.4 & 5.4 & 0.3 & 5.4 & 1.1 & 4.3 \\
    GENO CPU & 1.9 & 0.4 & 19.4 & 0.4 & 6.4 & 1.0 & 3.2 \\
    cuOSQP & 55.7 & failed & 206.6 & 22.1 & 163.6 & 32.9 & 36.1 \\
    SCS GPU & 2342.1 & 31.4 & N/A & 7995.0 & N/A & 1094.1 & 1227.1 \\
    \bottomrule
  \end{tabular}
  \caption{Running times in seconds for the dual SVM problem. All data sets were subsampled to $10,000$ data points. N/A indicates that the solver did not finish within $10,000$ seconds.}
  \end{center}
\end{table}

\begin{table}[h!]
\begin{center}
\small
  \begin{tabular}{l*{9}{r}}
    \toprule
    \multirow{2}{*}{Solver} & \multicolumn{8}{c}{dual SVM (a1a – a8a)} \\
    \cmidrule{2-9}
    & 1605 & 2265 & 3185 & 4781 & 6414 & 11220 & 16100 & 22696 \\
    \midrule
    this paper GPU & $0.1$ & $0.1$ & $0.1$ & $0.2$ & $0.2$ & $0.4$ & $0.7$ & $1.3$ \\
    this paper CPU & $0.0$ & $0.0$ & $0.1$ & $0.3$ & $0.5$ & $1.5$ & $3.2$ & $6.6$ \\
    GENO CPU & $0.0$ & $0.0$ & $0.1$ & $0.2$ & $0.4$ & $1.4$ & $3.0$ & $5.8$ \\
    cuOSQP & $1.0$ & $1.5$ & $3.0$ & $7.1$ & $13.5$ & $42.9$ & N/A & N/A \\
    SCS GPU & $2.6$ & $7.7$ & $40.0$ & $127.3$ & $312.9$ & $1565.9$ & N/A & N/A \\
    \bottomrule
  \end{tabular}
  \caption{Running times in seconds for dual SVM (adult data set, a1a – a8a).}
  \end{center}
\end{table}

\begin{table}[h!]
\begin{center}
\small
  \begin{tabular}{l*{8}{r}}
    \toprule
    \multirow{2}{*}{Solver} & \multicolumn{7}{c}{dual SVM (w1a – w7a)} \\
    \cmidrule{2-8}
    & 2477 & 3470 & 4912 & 7366 & 9888 & 17188 & 24692 \\
    \midrule
    this paper GPU & $0.4$ & $0.6$ & $0.7$ & $1.1$ & $1.5$ & $3.5$ & $7.0$ \\
    this paper CPU & $0.2$ & $0.3$ & $0.8$ & $2.0$ & $4.1$ & $13.6$ & $37.0$ \\
    GENO CPU & $0.1$ & $0.2$ & $0.5$ & $1.5$ & $2.6$ & $8.5$ & $22.4$ \\
    cuOSQP & $1.8$ & $3.5$ & $7.4$ & $19.8$ & $35.3$ & N/A & N/A \\
    SCS GPU & $14.9$ & $49.7$ & $147.9$ & $493.1$ & $1159.0$ & N/A & N/A \\
    \bottomrule
  \end{tabular}
  \caption{Running times in seconds for dual SVM (web data set, w1a – w7a).}
  \end{center}
\end{table}

\newpage

\subsection{Non-negative Least Squares}

 \begin{figure}[h]
  \centering
  \includegraphics[width=0.32\textwidth]{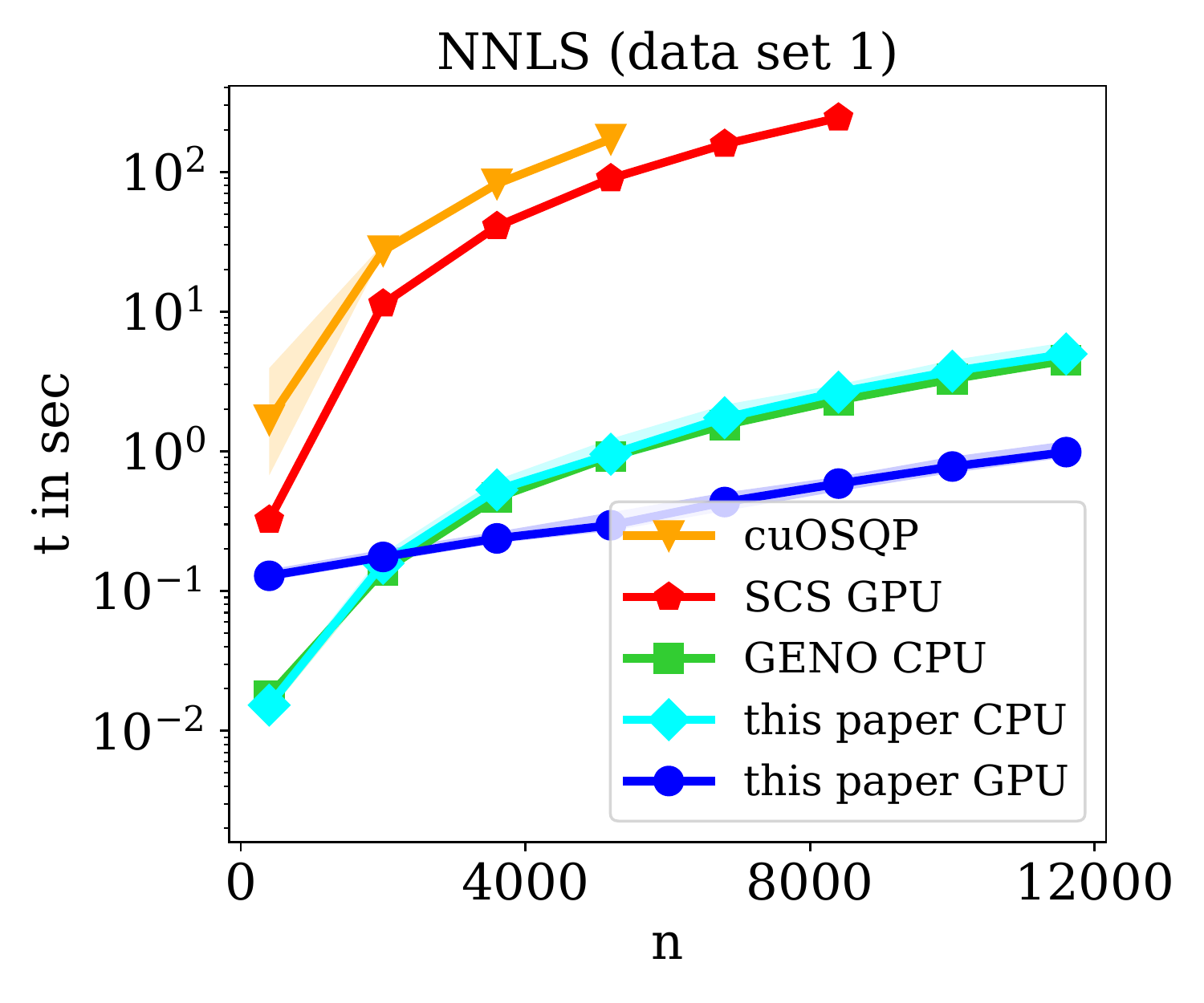}
  \includegraphics[width=0.32\textwidth]{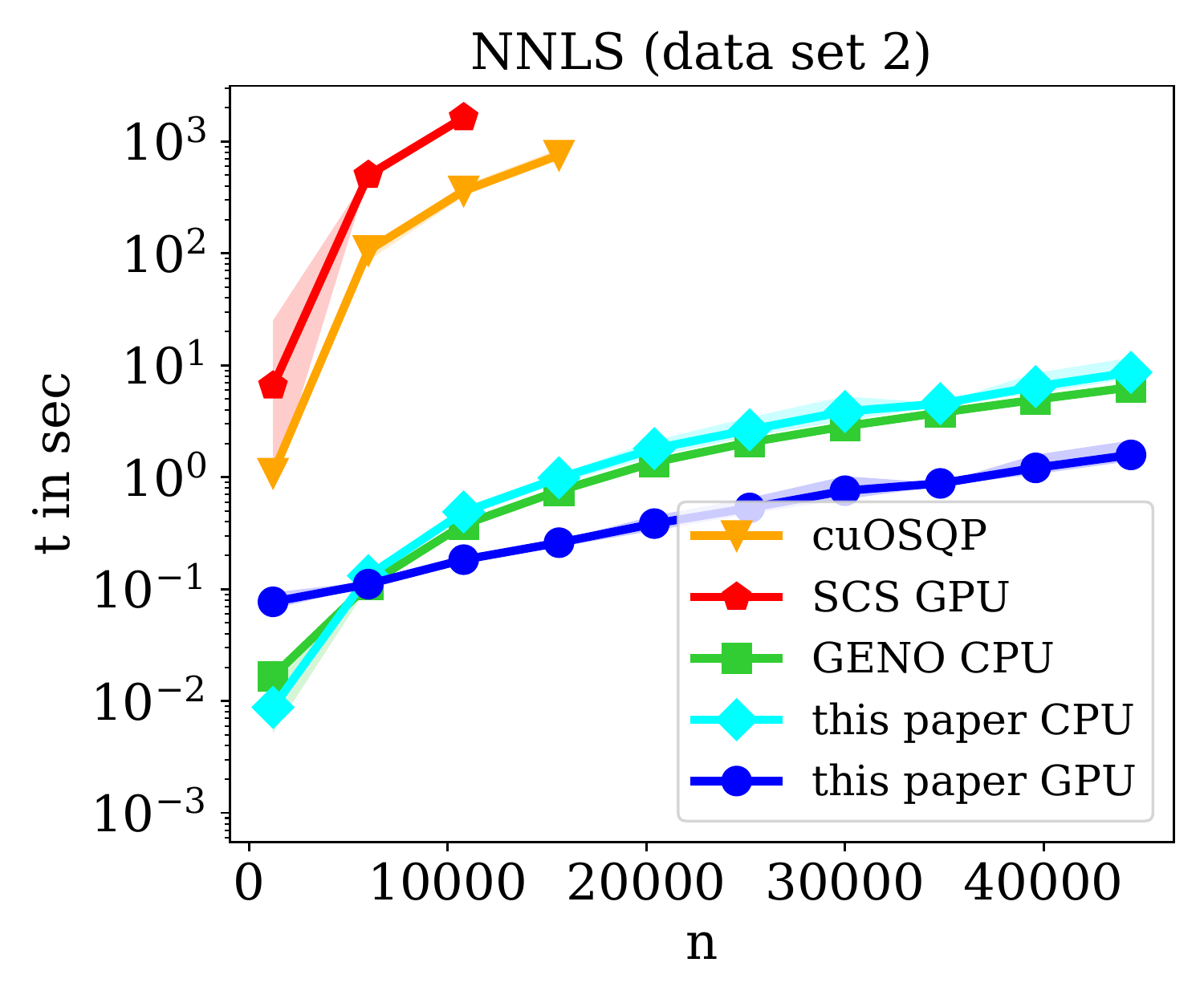}
  \caption{The plots show the running times for the non-negative least squares problem when run on the first and second data set.}
  \label{fig:NNLS}
\end{figure}

\begin{table}[h!]
\begin{center}
  \small
  \begin{tabular}{l*{7}{r}}
    \toprule
    \multirow{2}{*}{Solver} & \multicolumn{6}{c}{NNLS (data set 1)} \\
    \cmidrule{2-7}
                   & 3600 & 5200 & 6800 & 8400 & 10000 & 11600 \\
    \midrule
    this paper GPU & $0.2\pm 0.0$ & $0.3\pm 0.0$ & $0.4\pm 0.0$ & $0.6\pm 0.0$ & $0.8\pm 0.1$ & $1.0\pm 0.1$ \\
    this paper CPU & $0.5\pm 0.1$ & $1.0\pm 0.1$ & $1.7\pm 0.2$ & $2.6\pm 0.2$ & $3.7\pm 0.4$ & $5.0\pm 0.5$ \\
    GENO CPU       & $0.5\pm 0.0$ & $0.9\pm 0.0$ & $1.5\pm 0.0$ & $2.3\pm 0.0$ & $3.3\pm 0.1$ & $4.5\pm 0.1$ \\
    cuOSQP         & $82.0\pm 1.0$ & $171.4\pm 5.5$ & N/A & N/A & N/A & N/A \\
    SCS GPU        & $40.7\pm 0.4$ & $89.5\pm 1.0$ & $157.7\pm 2.0$ & $243.4\pm 3.3$ & N/A & N/A \\
    \bottomrule
  \end{tabular}
  \caption{Running times in seconds for NNLS (data set 1).}
  \end{center}
\end{table}

\begin{table}[h!]
\begin{center}
  \small
 
  \begin{tabular}{l*{7}{r}}
    \toprule
    \multirow{2}{*}{Solver} & \multicolumn{6}{c}{NNLS (data set 2)} \\
    \cmidrule{2-7}
                   & 6000             & 10800            & 15600          & 25200        & 34800        & 44400 \\
    \midrule                                                                                            
    this paper GPU & $0.1\pm 0.0$     & $0.2\pm 0.0$     & $0.3\pm 0.0$   & $0.5\pm 0.1$ & $0.9\pm 0.0$ & $1.6\pm 0.2$ \\
    this paper CPU & $0.1\pm 0.0$     & $0.5\pm 0.0$     & $1.0\pm 0.1$   & $2.7\pm 0.4$ & $4.5\pm 0.1$ & $8.6\pm 1.3$ \\
    GENO CPU       & $0.1\pm 0.0$     & $0.4\pm 0.0$     & $0.8\pm 0.0$   & $2.1\pm 0.0$ & $3.8\pm 0.0$ & $6.4\pm 0.1$ \\
    cuOSQP         & $106.3\pm 10.5$  & $362.3\pm 27.7$  & $755.4\pm 41.1$ & N/A & N/A & N/A  \\
    SCS GPU        & $500.2\pm 4.2$   & $1637.4\pm 14.3$ & N/A & N/A & N/A & N/A \\
    \bottomrule
  \end{tabular}
  \caption{Running times in seconds for NNLS (data set 2).}
  \end{center}
\end{table}

\newpage
\subsection{Joint Probability}
Here, we show the running times for computing the joint probability distribution of two probability distributions, i.e., a distribution that has the two given distributions as marginals, see also the main paper. The second data set was created in the same way as in~\cite{FrognerP19}, i.e., we sampled $u$ and $v$ uniformly at random from the unit interval and scaled them such that each vector sums up to one. All entries of the cost matrices $M$ were also sampled uniformly at random from the unit interval. We fixed the regularization parameter $\lambda =\frac{1}{2}$. We set the size of the problems to be $m=2n$. Hence, when $n=1000$, the corresponding optimization problem involves $2\cdot 10^6$ optimization variables and has $3000$ constraints. Figure~\ref{fig:jointProb} shows the running times for both data sets and both regularizers for varying problem sizes, including the running times for the CPU versions.

 \begin{figure}[h]
  \centering
  \includegraphics[width=0.24\textwidth]{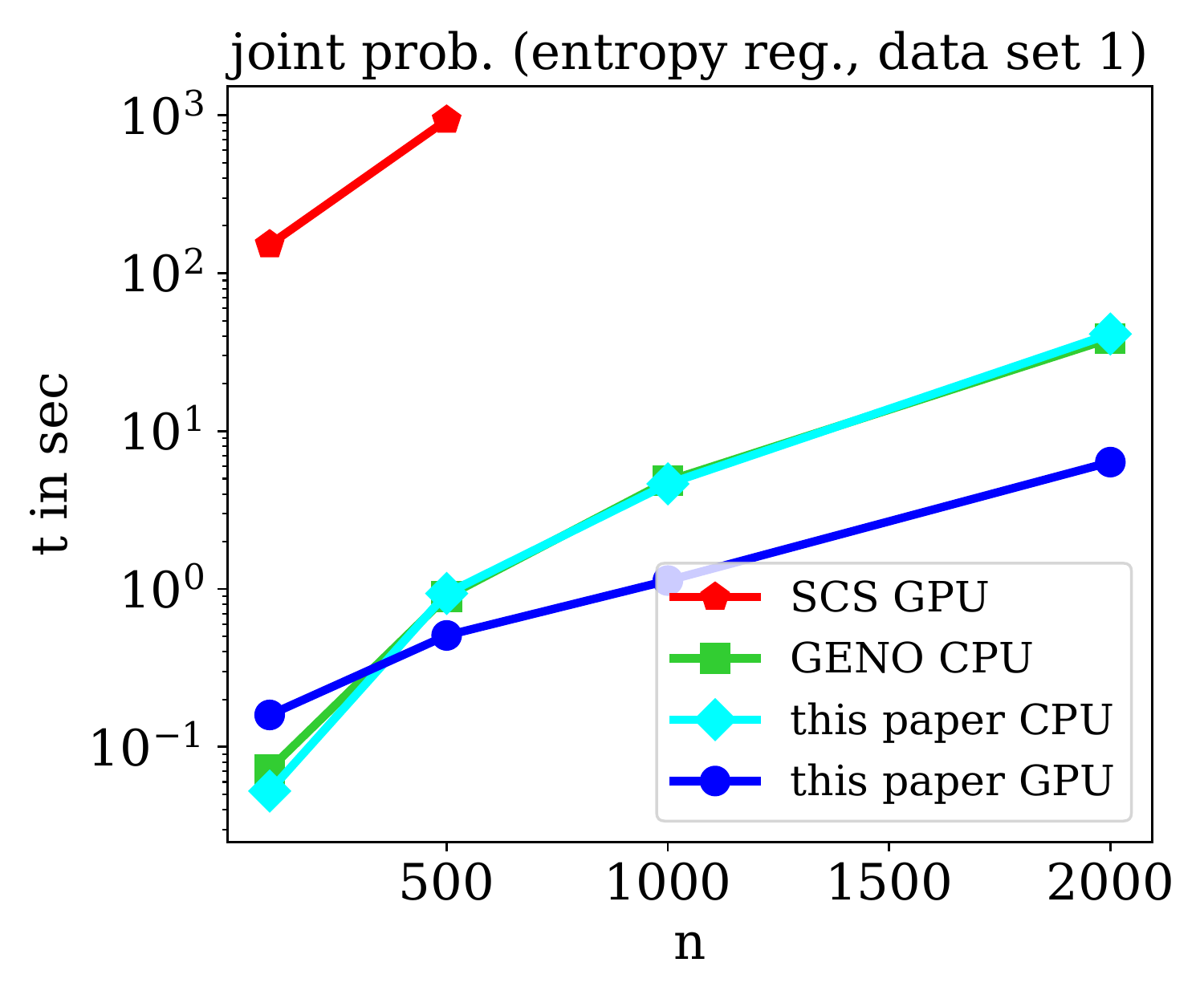}
  \includegraphics[width=0.24\textwidth]{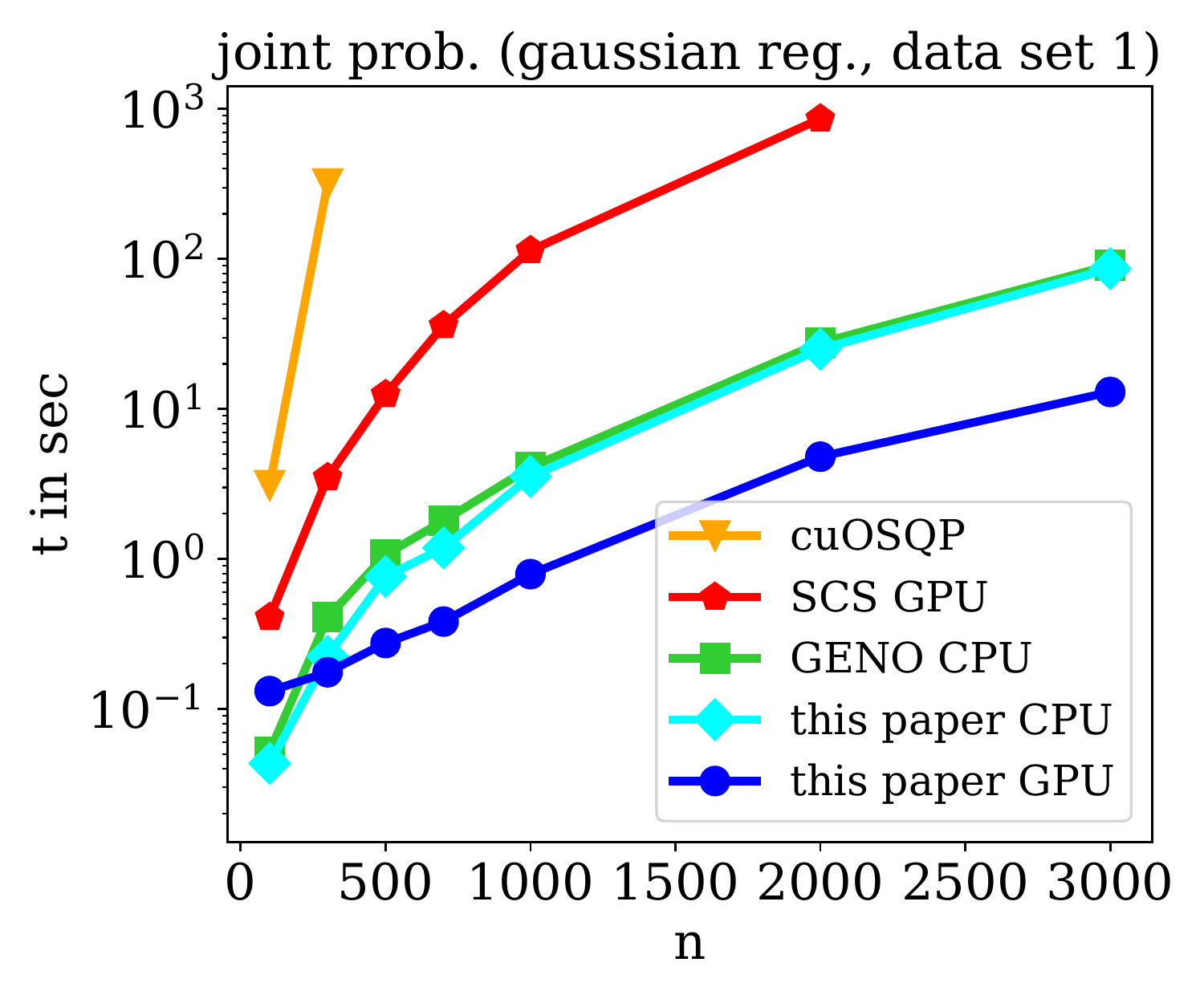}
  \includegraphics[width=0.24\textwidth]{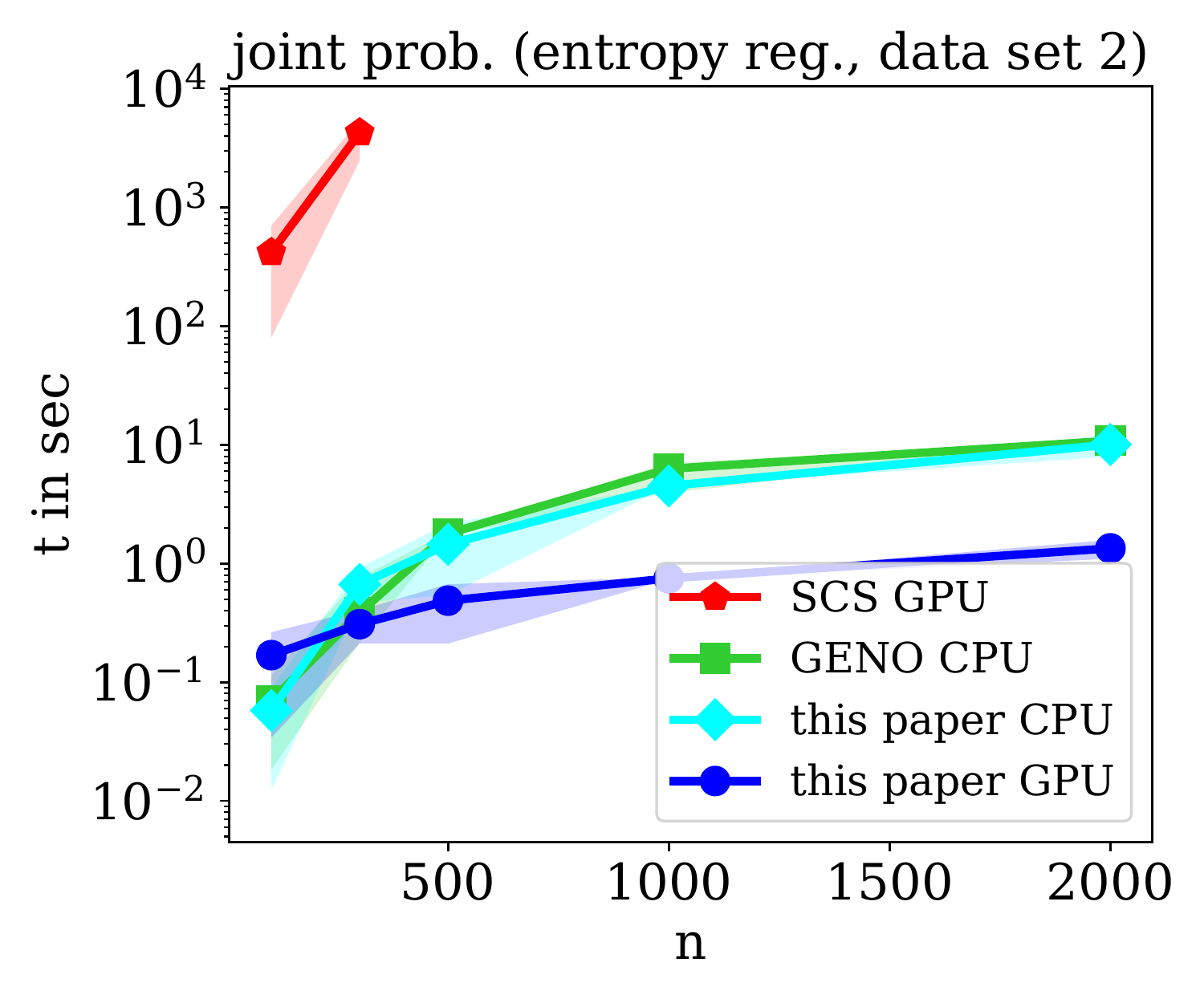}
  \includegraphics[width=0.24\textwidth]{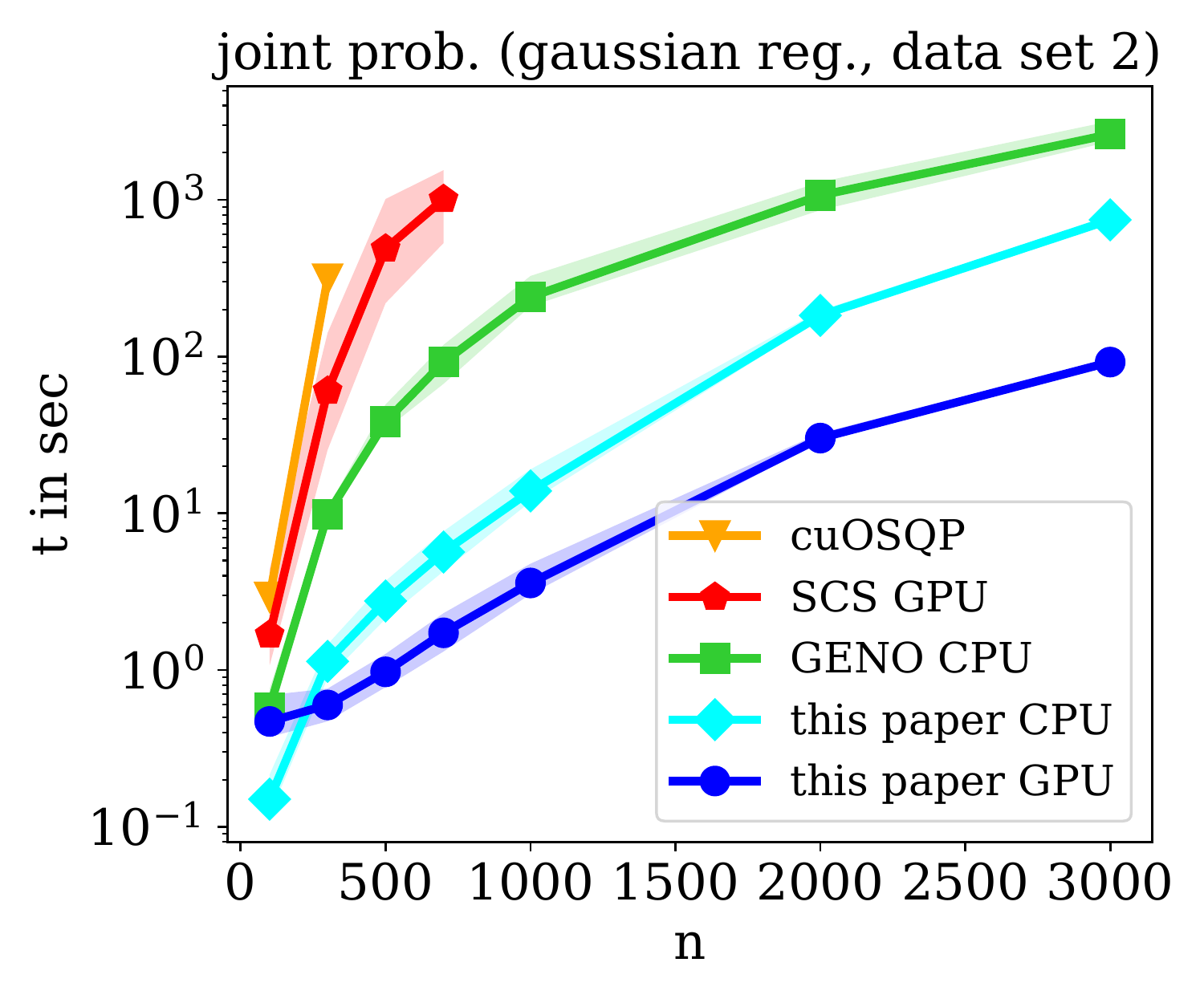}
  \caption{Running times for computing a joint probability distribution from two marginal distributions. The two plots on the left show the running times for the first data set and the two plots on the right the running times for the second data set. For each data set, we show the running times when the entropy prior is used and when the Gaussian prior is used. Note, that the cuOSQP solver cannot solve the problem when the entropy prior is used. Since the second data set is drawn from a random distribution, we ran all solvers ten times and report the mean running time and corresponding error bars.}
  \label{fig:jointProb}
\end{figure}

\begin{table}[h!]
\begin{center}
  \small
  \begin{tabular}{l*{5}{r}}
    \toprule
    \multirow{2}{*}{Solver} & \multicolumn{4}{c}{joint prob. (entropy reg., data set 1)} \\
    \cmidrule{2-5}
    & 100 & 500 & 1000 & 2000 \\
    \midrule
    this paper GPU & $0.2$ & $0.5$ & $1.1$ & $6.4$ \\
    this paper CPU & $0.1$ & $0.9$ & $4.6$ & $41.2$ \\
    GENO CPU & $0.1$ & $0.9$ & $4.8$ & $38.5$ \\
    SCS GPU & $151.8$ & $933.5$ & N/A & N/A \\
    \bottomrule
  \end{tabular}
  \caption{Running times in seconds for joint prob. (entropy reg., data set 1).}
  \end{center}
\end{table}

\begin{table}[h!]
\begin{center}
  \small
  \begin{tabular}{l*{8}{r}}
    \toprule
    \multirow{2}{*}{Solver} & \multicolumn{7}{c}{joint prob. (gaussian reg., data set 1)} \\
    \cmidrule{2-8}
    & 100 & 300 & 500 & 700 & 1000 & 2000 & 3000 \\
    \midrule
    this paper GPU & $0.1$ & $0.2$ & $0.3$ & $0.4$ & $0.8$ & $4.8$ & $13.0$ \\
    this paper CPU & $0.0$ & $0.2$ & $0.8$ & $1.2$ & $3.6$ & $25.1$ & $86.2$ \\
    GENO CPU & $0.1$ & $0.4$ & $1.1$ & $1.8$ & $4.1$ & $27.7$ & $90.9$ \\
    cuOSQP & $3.1$ & $318.8$ & N/A & N/A & N/A & N/A & N/A \\
    SCS GPU & $0.4$ & $3.5$ & $12.5$ & $36.0$ & $114.1$ & $861.4$ & N/A \\
    \bottomrule
  \end{tabular}
  \caption{Running times in seconds for joint prob. (gaussian reg., data set 1).}
  \end{center}
\end{table}

\begin{table}[h!]
\begin{center}
  \small
  \begin{tabular}{l*{6}{r}}
    \toprule
    \multirow{2}{*}{Solver} & \multicolumn{5}{c}{joint prob. (entropy reg., data set 2)} \\
    \cmidrule{2-6}
    & 100 & 300 & 500 & 1000 & 2000 \\
    \midrule
    this paper GPU & $0.2\pm 0.1$ & $0.3\pm 0.1$ & $0.5\pm 0.2$ & $0.7\pm 0.0$ & $1.3\pm 0.2$ \\
    this paper CPU & $0.1\pm 0.0$ & $0.7\pm 0.1$ & $1.4\pm 0.6$ & $4.5\pm 0.0$ & $10.1\pm 1.5$ \\
    GENO CPU & $0.1\pm 0.0$ & $0.4\pm 0.2$ & $1.8\pm 0.0$ & $6.3\pm 0.9$ & $10.8\pm 0.0$ \\
    SCS GPU & $417.9\pm 257.7$ & $4267.7\pm 1408.3$ & N/A & N/A & N/A \\
    \bottomrule
  \end{tabular}
  \caption{Running times in seconds for joint prob. (entropy reg., data set 2).}
  \end{center}
\end{table}

\begin{table}[h!]
\begin{center}
  \small
  \begin{tabular}{l*{6}{r}}
    \toprule
    \multirow{2}{*}{Solver} & \multicolumn{5}{c}{joint prob. (gaussian reg., data set 2)} \\
    \cmidrule{2-6}
                   & 300           & 700 & 1000 & 2000 & 3000 \\
    \midrule
    this paper GPU & $0.6\pm 0.1$    & $1.7\pm 0.3$ & $3.6\pm 0.6$ & $30.2\pm 1.5$ & $92.3\pm 1.9$ \\
    this paper CPU & $1.1\pm 0.2$    & $5.7\pm 1.0$ & $13.9\pm 2.3$ & $182.9\pm 9.1$ & $746.1\pm 21.9$ \\
    GENO CPU       & $9.9\pm 0.6$    & $92.2\pm 13.2$ & $240.1\pm 35.3$ & $1068.0\pm 153.7$ & $2622.4\pm 236.7$ \\
    cuOSQP         & $314.0\pm 7.8$  & N/A & N/A & N/A & N/A \\
    SCS GPU        & $60.4\pm 39.4$  & $1009.3\pm 355.1$ & N/A & N/A & N/A \\
    \bottomrule
  \end{tabular}
  \caption{Running times in seconds for joint prob. (gaussian reg., data set 2).}
  \end{center}
\end{table}

\end{document}